\documentclass[letterpaper]{article} 
\usepackage{aaai25}  
\usepackage{times}  
\usepackage{helvet}  
\usepackage{courier}  
\usepackage[hyphens]{url}  
\usepackage{graphicx} 
\usepackage{natbib}  
\usepackage{caption} 
\usepackage{algorithm}
\usepackage[noend]{algorithmic}

\usepackage{newfloat}
\usepackage{listings}
\DeclareCaptionStyle{ruled}{labelfont=normalfont,labelsep=colon,strut=off} 
\floatstyle{ruled}
\newfloat{listing}{tb}{lst}{}
\floatname{listing}{Listing}
\usepackage{amsmath,amssymb,amsthm,bm,amssymb,dsfont}
\usepackage{booktabs}
\usepackage{multirow}

\def\dataset{{\mathcal{D}}}

\def\servedataset{{\hat \dataset}}
\def\servedatasetclient{\dataset^{s}_i}

\def\E{\mathbb{E}}
\def\requestrate{\lambda^{a}}
\def\maxtransrate{\mu^{\mathrm{max}}}
\def\transrate{\lambda^{t}}
\def\fraction{f}
\def\optimprob{{\mathbb{P}}}

\def\serverate{\lambda^s}

\def\setnodes{\mathcal{N}}
\def\setleafnodes{\mathcal{L}}
\def\setchildnodes{\setnodes^-}

\def\setnodesexit{\mathcal{N}_e}
\def\model{\pmb{w}}

\def\reals{\mathbb{R}}

\newcommand{\set}[1]{\left\{#1\right\}}
\newcommand{\brackets}[1]{\left[#1\right]}
\newcommand{\parentheses}[1]{\left(#1\right)}

\usepackage{mathtools}

\newcommand{\norm}[1]{\left\lVert#1\right\rVert}

\newtheorem{lemma}{Lemma}
\DeclareMathOperator{\VarVec}{{\mathbb{V}ar}}

\newcommand\scalar[2]{\langle #1, #2 \rangle}
\newcommand{\ww}{\boldsymbol{w}}
\newcommand{\vv}{\boldsymbol{v}}
\newcommand{\LL}{\boldsymbol{\Lambda}}
\newcommand{\LLa}{\boldsymbol{\tilde\Lambda}}   
\newcommand{\wtjce}{\ww^{(t,j)}_{i,e}} 
\newcommand{\Fce}{F_{i,e}}  
\newcommand{\FSb}{F_{S,\LLa}} 
\newcommand{\Fu}{F_{\mathcal{D},\vphantom{\LLa}\LL}}    
\newcommand{\Fb}{F_{\mathcal{D},\LLa}}    
\newcommand{\Ce}{\mathcal C_e} 
\newcommand{\Sep}{S_{e,\boldsymbol{p}}} 
\newcommand{\FSe}{F_{\Sep}}
\newcommand{\bigO}{\mathcal{O}}
\newcommand{\VC}{\textrm{VCdim}}
\newcommand{\Fe}{F_{\mathcal{D},e}} 
\DeclareMathOperator{\diam}{diam}

\DeclareMathOperator{\tv}{dist_{TV}}
\DeclareMathOperator{\Pd}{Pdim}

\DeclareMathOperator{\Var}{Var}

\newtheorem{theorem}{Theorem}
\newtheorem{assumption}{Assumption}

\usepackage{xcolor}

\nocopyright
\title{Federated Learning for Collaborative Inference Systems: \\The Case of Early Exit Networks}

\author {
    Caelin Kaplan\textsuperscript{\rm 1, \rm 2},
    Angelo Rodio\textsuperscript{\rm 2},
    Tareq Si Salem\textsuperscript{\rm 2, \rm 3},
    Chuan Xu\textsuperscript{\rm 2},
    Giovanni Neglia\textsuperscript{\rm 2}
}
\affiliations {
    \textsuperscript{\rm 1}SAP Labs France. Email: \{firstname.lastname\}@sap.com\\
    \textsuperscript{\rm 2}Inria, Université Côte d’Azur. Email: \{firstname.lastname\}@inria.com\\
    \textsuperscript{\rm 3}Huawei. Email: \{firstname.lastname\}@huawei.com
}

\begin{document}

\maketitle

\begin{abstract}
As Internet of Things (IoT) technology advances, end devices like sensors and smartphones are progressively equipped with AI models tailored to their local memory and computational constraints. Local inference reduces communication costs and latency; however, these smaller models typically underperform compared to more sophisticated models deployed on edge servers or in the cloud. Collaborative Inference Systems (CISs) address this performance trade-off by enabling smaller devices to offload part of their inference tasks to more capable devices. These systems often deploy hierarchical models that share numerous parameters, exemplified by deep neural networks that utilize strategies like early exits or ordered dropout. In such instances, Federated Learning (FL) may be employed to jointly train the models within a CIS. Yet, traditional training methods have overlooked the operational dynamics of CISs during inference, particularly the potential high heterogeneity in serving rates across nodes. To address this gap, we propose a novel FL approach designed explicitly for use in CISs that accounts for these variations in serving rates. Our framework not only offers rigorous theoretical guarantees but also surpasses state-of-the-art training algorithms for CISs, especially in scenarios where end devices handle higher inference request rates and where data availability is uneven among nodes.
\end{abstract}


\section{Introduction}
The integration of intelligent capabilities into devices such as sensors, smartphones, and IoT equipment is rapidly increasing~\cite{ren2023survey,campolo2023network}. Despite these advancements, a significant hurdle in this field is resource heterogeneity within real-world networks, where nodes often have varying memory and computational capacities. This disparity makes it infeasible to deploy a uniform AI model across all network nodes~\cite{lim2020federated,kairouz2021advances}. To address this issue, Collaborative Inference Systems (CISs) have been proposed~\cite{he2021,yang2022,salem2023toward,ren2023survey}, which allow less capable devices to offload a portion of their inference tasks to more powerful devices within the network. 

While most existing research on CISs assumes these AI models are already trained and focuses on either optimizing their placement within networks and/or developing collaborative serving policies~\cite{li2019edge,zeng2019boomerang,salem2023toward,ren2023survey,jankowski2023}, significantly less attention has been given to training methodologies within a CIS. We address this gap by focusing on scenarios where models are collaboratively trained on distributed datasets hosted by the nodes (i.e., devices) that will later perform the inference tasks.

Federated Learning (FL)~\cite{mcmahan2017communication,li2020federated,kairouz2021advances} provides a framework for such collaborative training, enabling nodes to train machine learning models without sharing their local data. In FL, knowledge transfer among heterogeneous models can be achieved through explicit knowledge distillation---which typically requires a public dataset~\cite{lin2020ensemble, mora2022knowledge}---or by having the models share a subset of parameters. Within this latter approach, the most common method is to jointly train Deep Neural Networks (DNNs) that either share entire layers or specific parameters within a layer. Techniques like Ordered Dropout, which selectively drops parts of the network during training~\cite{diao2020heterofl, horvath2021fjord}, and Early Exit Networks, which allow models to make predictions at intermediate layers~\cite{teerapittayanon2016branchynet, teerapittayanon2017distributed}, can be used to customize models according to the varying memory and computational constraints of different nodes.

However, optimizing these shared parameters is challenging because different models may need distinct representations from the same layer to achieve optimal performance. For example, in an early exit network, a shallow model may need the layers just before its classifier to focus on classification, while a deeper model may instead rely on these layers to extract basic feature representations. A crucial consideration in this optimization process for CISs is the role each model plays during inference. Models handling a higher volume of inference requests should exert greater influence on the learning of shared parameters. This ensures that the most frequently requested models are better optimized, thereby enhancing overall inference performance.

Despite its importance, previous research has largely overlooked this unique challenge within CISs. Existing training methods, particularly those designed for distributed early exit networks~\cite{teerapittayanon2017distributed,nawar2023fed,ilhan2023scalefl}, treat all models equally,
failing to account for the heterogeneity in model capacities and performance.
Only a few studies have empirically suggested assigning weights based on model complexity~\cite{DBLP:conf/aaai/HuDHB19,kaya2019shallow}, but they still disregard the corresponding inference request rates.
To bridge this gap, our paper introduces a 
theoretically grounded FL training algorithm specifically designed to improve the overall CIS inference performance.

\paragraph{Contributions.} 
\begin{enumerate}
    \item We formalize the first inference-aware FL training framework for CISs, with the goal of maximizing overall inference accuracy. We define our objective function as a weighted sum of the expected losses for each model, where the weights, $\LL$, correspond to the expected future inference request rates.
    \item We propose a novel and practical inference-aware FL training algorithm designed for CISs. Our algorithm minimizes the weighted sum of empirical losses across nodes, using input weights $\LLa$ that may differ from the expected $\LL$. Moreover, it enables computationally stronger nodes to assist weaker ones in model training, according to predefined probabilities $\boldsymbol{p}$.
    \item We rigorously analyze the impact of the key parameters $\LLa$ and $\boldsymbol{p}$ on the generalization error, optimization error, and bias error, providing a deeper understanding of how these factors affect the overall training process and final inference performance. From this theoretical analysis, we derive practical configuration guidelines for our proposed training algorithm.
    \item We evaluate the effectiveness of our algorithm, showing that it significantly outperforms state-of-the-art methods, particularly in realistic scenarios where end devices handle higher inference request rates.
\end{enumerate}

\section{Background and Related Work}\label{sec:background}
In this section, we discuss the relevant background necessary to understand CISs, FL, and Early Exit Networks.

\subsection{Collaborative Inference Systems}\label{sec:cis-background}
Collaborative Inference Systems (CISs)~\cite{ren2023survey}, also known in the literature as Inference Delivery Networks \cite{salem2023toward}, enable smaller devices to offload part of their inference tasks to more capable devices, and represent an active field of study. The scope of collaboration in these systems may vary, extending beyond the traditional device-cloud model, to include intermediate nodes such as edge servers, regional clouds, or a collective of devices within direct transmission range of each other~\cite{teerapittayanon2017distributed, li2019edge, zeng2019boomerang, ren2023survey, salem2023toward}. 
While collaboration within a CIS can take many forms~\cite{matsubara2022split,malka2022decentralized, yilmaz2022over,salem2023toward}, in this paper, we focus on a hierarchical structure of nodes, each equipped with increasingly complex models that collaborate by forwarding inference requests to more powerful nodes within the network.
Although much previous work has focused on optimizing the deployment and utilization of already trained models in a CIS~\cite{li2019edge, zeng2019boomerang, salem2023toward,jankowski2023},
our research shifts focus to the less-explored challenge of \emph{training} these models in a FL context. 

\subsection{Federated Learning for a CIS} \label{sec:fl-background}
Traditional FL algorithms (e.g., FedAvg~\cite{mcmahan2017communication}, FedProx~\cite{li2020federatedprox}) typically assume that the participating nodes
have identical storage and computation capacities, meaning that each node holds a DNN of the same architecture and can perform an equal amount of computation during training. 
However, recent algorithms have been developed to efficiently train multiple models of different sizes within a network, with the most practical approach being the joint training of models that share a subset of parameters.
For instance, FjORD~\cite{horvath2021fjord} introduces a framework where a DNN is pruned by channels to generate nested submodels of different sizes that can fit into heterogeneous nodes, following a mechanism known as ordered dropout. 
A similar idea is explored in HeteroFL~\cite{diao2020heterofl}. 
Alternative approaches involve the use of early exit networks~\cite{nawar2023fed} or a combination of these two methods~\cite{ilhan2023scalefl}. While our algorithm and analysis apply to both pruning (e.g., ordered dropout) and early exit strategies, we focus on early exit networks for clarity and concreteness. Early exit networks also offer the clearest example of collaborative inference, as weaker nodes forward intermediate representations to more powerful nodes, unlike in FjORD/HeteroFL, where the input is forwarded.

\subsection{Early Exit Networks}\label{sec:ee-background}
 \begin{figure}[h]
     \centering
     \includegraphics[clip,scale=0.25]{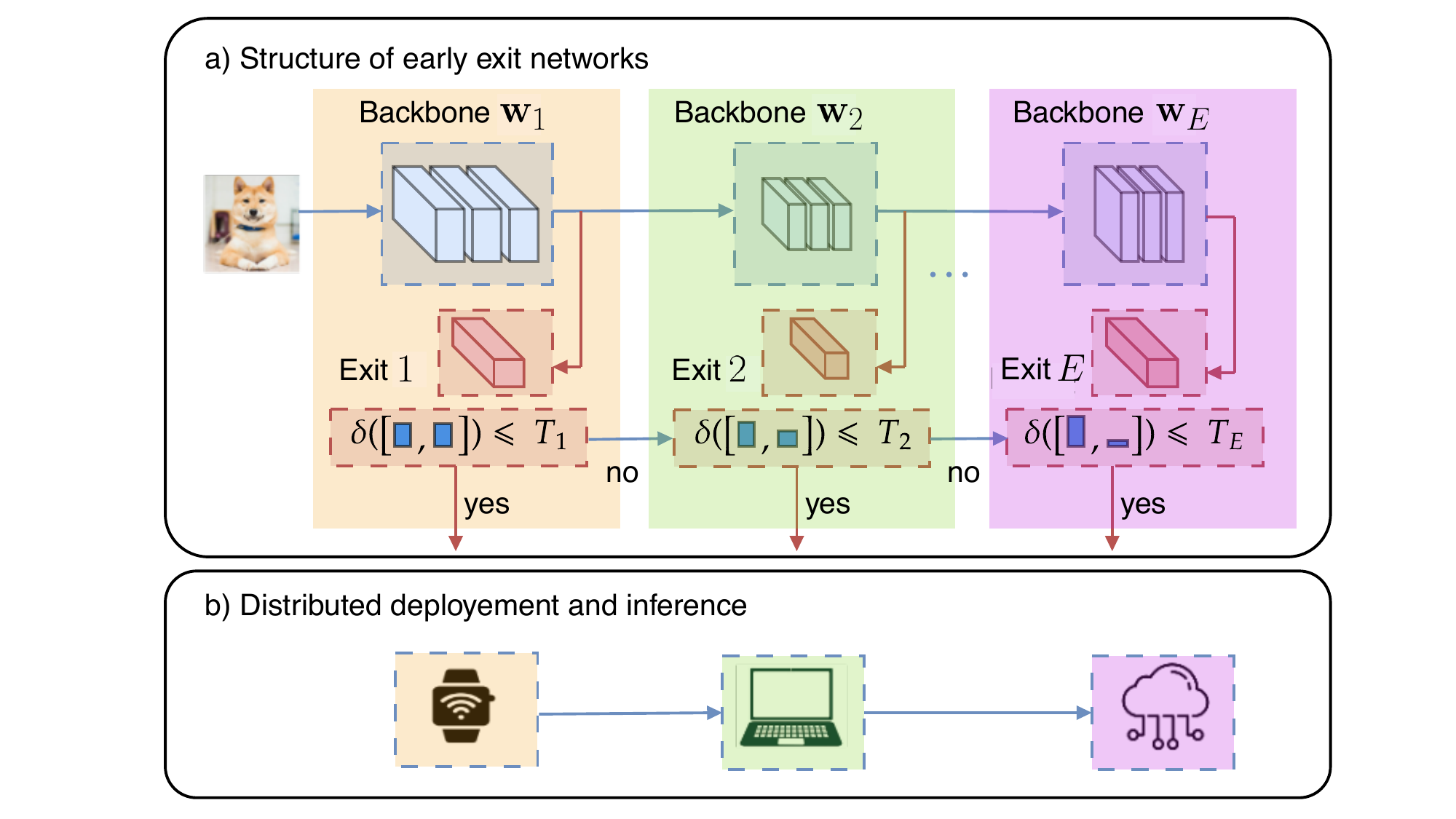}
     \caption{Early Exit Networks for Collaborative Inference System. An input sample is first passed through the initial layers of the DNN until it reaches Exit $1$. If the measure of prediction uncertainty is below the threshold $T_1$, the prediction is served at the current node. Otherwise, the intermediate representation of the current input is transferred to a node with greater computational capacity, and inference continues. This process repeats until the prediction uncertainty is below $T_e$ or the final Exit $E$ is reached. 
     }
     \label{fig:early_exit_classic}
 \end{figure}
Early Exit Networks (EENs), introduced initially  as BranchyNets~\cite{teerapittayanon2016branchynet}, extend DNNs by adding classifiers, or early exits, at intermediate layers. 
For instance, integrating an early exit into a standard ResNet-34~\cite{he2016deep} architecture might involve adding a classifier after the 8th residual block, thereby creating a smaller network within the original ResNet-34 that has the same depth as a ResNet-18.
The initial motivation for this design is to enable faster inference with less computational cost, which is especially useful in computationally heavy computer vision and natural language processing tasks~\cite[Table~4]{matsubara2022split}. 
Figure~\ref{fig:early_exit_classic}(a) details the inference process in standard EENs.
The typical training procedure involves minimizing the expected weighted loss across \emph{all} exits~\cite{teerapittayanon2016branchynet,DBLP:conf/iclr/HuangCLWMW18, li2019improved, DBLP:conf/aaai/HuDHB19, kaya2019shallow}:
\begin{equation}\label{eq:classic-earlyexit}
    \min_{\model\in \reals^d}\E_{z\sim \dataset}\brackets{\: \sum_{e\in\mathcal{E}} \alpha_e \: \ell^{(e)}(\ww,z)},
\end{equation}
where each data sample $z = \parentheses{x, y}$ is drawn from the data distribution $\dataset$, with $x$ as the input features and $y$ the corresponding target, $\model$ represents the EEN parameters, $\mathcal{E}=\{1,\dots,E\}$ is the set of early exits, $\ell^{(e)}$ the loss at the $e$-th exit, and $\alpha_e \in \reals_{\geq 0}$ is the weight assigned to $e$-th exit's loss. In a CIS, this process extends to a distributed setting where each node holds a model with its assigned exit and all earlier exits. During inference, the intermediate representation can be sent to more powerful nodes, as shown in Fig.~\ref{fig:early_exit_classic}(b). 

The weight coefficients $\alpha_e$ are crucial in determining each exit's ($e\in\mathcal{E}$) contribution to the overall model performance and can be assigned in various ways. Traditional training approaches generally assign equal weights to all exits~\cite{teerapittayanon2017distributed, DBLP:conf/iclr/HuangCLWMW18,nawar2023fed,ilhan2023scalefl}.
We refer to these methods collectively as the ``Equal Weight'' strategy. 
Alternatively, more complex approaches have been considered that allocate weights in proportion to each exit's computational complexity, often measured in FLOPS, which results in assigning more weight to later exits~\cite[``Linear'' baseline]{kaya2019shallow,DBLP:conf/aaai/HuDHB19}. We refer to these methods collectively as the "FLOPS Prop" strategy. 
These existing approaches overlook the fact that inference request rates can vary significantly across real-world networks, leading to accuracy drops in likely scenarios where end devices (e.g., smartphones) with shallow models handle most of the requests. 
Our proposed method addresses this issue by systematically incorporating these varying request rates into the training process.


\section{Federated Early Exit Networks for CISs}
\label{sec:algorithm-proposed}
In this section, we formalize the CIS problem~(Sec.~\ref{subsec:formulation}), present our collaborative training algorithm~(Sec.~\ref{subsec:algorithm}), provide theoretical convergence guarantees~(Sec.~\ref{subsec:theory}), and propose practical analysis-driven configuration rules~(Sec.~\ref{subsec:proposed-method}).

\subsection{Problem formulation}\label{subsec:formulation}

\paragraph{Network Topology.} In a CIS, the network is composed of a set of nodes $\setnodes = \set{1, 2, \dots, N}$, organized in a hierarchical tree structure such as a cloud-edge-device model~\cite{ren2023survey}, where parent nodes possess greater computational resources and memory than their child nodes. 
While we consider a tree topology for presentation purposes, 
we stress that 
the proposed algorithm (Sec.~\ref{subsec:algorithm}) and its theoretical guarantees (Sec.~\ref{subsec:theory}) are broadly applicable to any directed acyclic graph, regardless of where the most powerful nodes are positioned within the network.

Leaf nodes, which have no children, are represented by the set $\setleafnodes\subset \setnodes$, and each node $i$ has a set of child nodes, denoted by $\setchildnodes_i$. 
During \emph{training}, each node $i$ holds multiple early exits, up to a maximum exit $E_i\leq E$, with the constraint that $E_i>E_j, \forall j\in \setchildnodes_i$. However, during \emph{inference}, node $i$ utilizes only its largest exit $E_i$ to ensure the most accurate prediction. The set $\mathcal{N}_e$ denotes nodes that use early exit $e$ for inference, i.e.,~$\mathcal{N}_e=\{i\in \setnodes \mid E_i = e\}$.  

\paragraph{Real-time Inference Requests.}\label{sec:inference-requests}
\begin{figure}[H]
\centering
\includegraphics[width=0.43\textwidth]{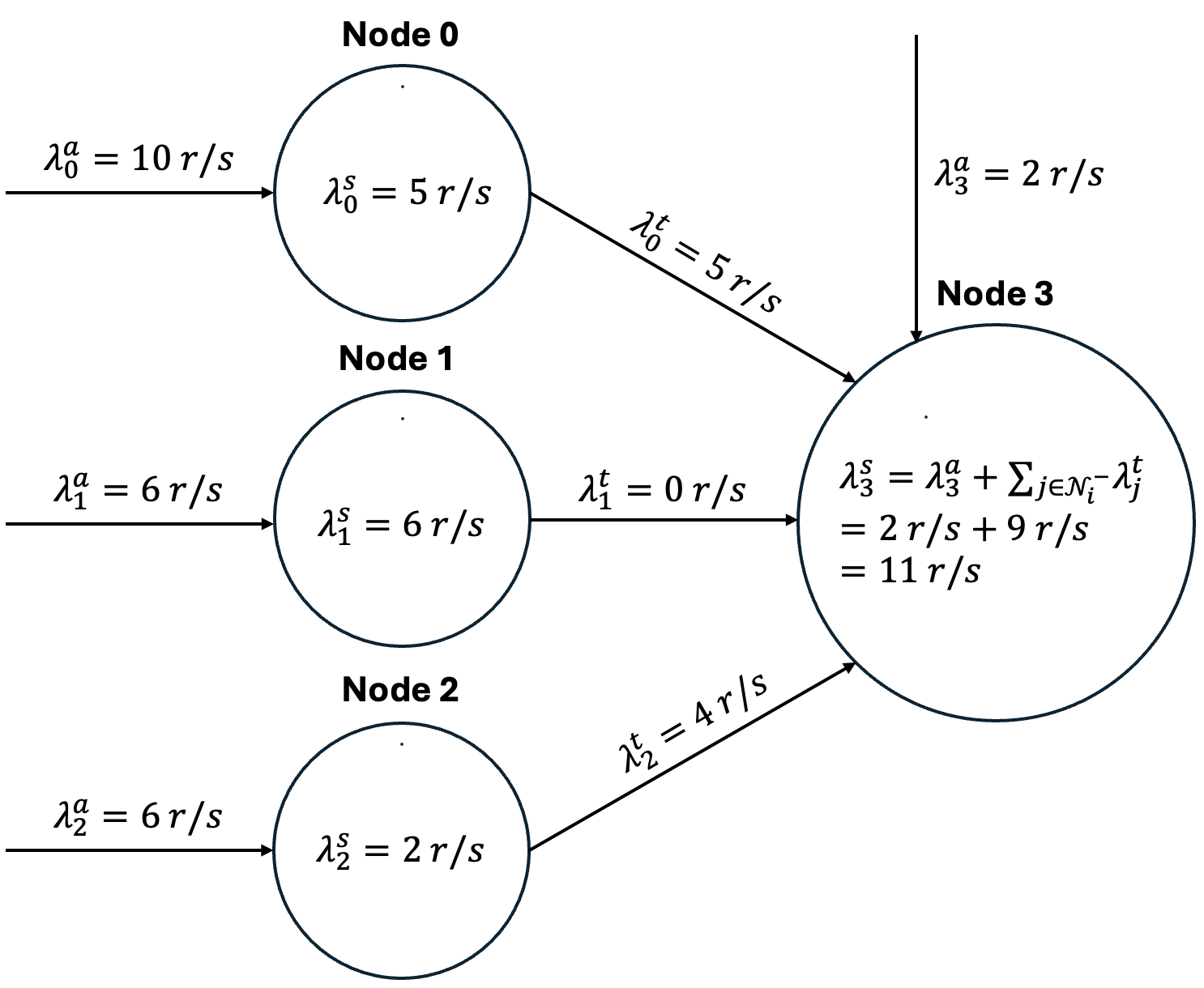}
\caption{
An example of a two-layer network with four nodes: Node~$0$, Node~$1$, and Node~$2$ each receive local requests, $\requestrate_i$ (in requests per second, r/s), serve a portion locally, $\serverate_{i}$, and transfer the remainder, $\transrate_{i}$, to their parent. Node~$3$ receives requests both locally and from its children, and serves all requests as it has no parent.
} 
\label{fig:inference-requests}
\end{figure}

Local inference requests arrive at each node~$i\in\setnodes$ with an arrival rate $\requestrate_i \in \reals_{\geq 0}$. A child node $i$ can transfer inference requests to its parent node with a transfer rate $\transrate_{i}$. 
The total requests at node $i$ include both its local requests and those transferred from its children.
Each node $i$ then serves a fraction $\fraction_{i} \in [0,1]$ of these requests locally using its largest exit $E_i$, resulting in a serving rate $\serverate_{i}$:
\begin{align}
    \textstyle\serverate_{i} \triangleq \left(\requestrate_i+ \sum_{j\in \setchildnodes_i} \transrate_{j}\right)\fraction_i, \label{eq:served}
\end{align}
while remaining requests are transferred to the parent node:
\begin{align}
    \transrate_{i} \triangleq  \textstyle \left(\requestrate_i+ \sum_{j\in \setchildnodes_i} \transrate_{j}\right)(1-\fraction_i). \label{eq:transmission}
\end{align}
Fig.~\ref{fig:inference-requests} presents a straightforward numerical example illustrating how a CIS manages inference requests.

The transfer rate $\transrate_{i}$ is constrained by an upper limit $\maxtransrate_{i}$, determined by the network's upstream bandwidth or the target inference delays.
Each node is aware of its maximum transfer rate $\maxtransrate_i$ and an estimate of its local arrival rate $\requestrate_i$.
Furthermore, nodes rank incoming samples by difficulty, allowing them to select the fraction $f_i$ of most favorable samples to serve locally~\cite{teerapittayanon2016branchynet, DBLP:conf/iclr/HuangCLWMW18,kaya2019shallow}.
The data distribution of these served samples at node $i$ is $\mathcal{D}^s_i$. 

\paragraph{Training objective for CISs.} 
The primary goal of training in a CIS is to minimize the total loss across all served samples throughout the network,  maximizing inference quality.
We formalize this objective as the first \emph{inference-aware} training framework for CISs using EENs, where the optimization problem is defined over the model parameters $\ww\in \mathcal{W}$ and the serving fractions $\{\fraction_i\}$ for each node:  
\begin{align}
        &\optimprob_1:&\phantom{} &\min_{\ww \in \mathcal{W}, \{\fraction_i\}} \sum_{i \in \setnodes} \serverate_i ~\E_{z\sim \servedatasetclient}  \brackets{ \ell^{(E_i)}( \ww, z)}, \nonumber  &\\
                &\text{s.t.,} & & \transrate_i\leq \maxtransrate_{i},~f_i \in [0,1], ~\text{Eqs.}~\ref{eq:served} \text{ and }~\ref{eq:transmission}, \; \forall i\in \setnodes. \label{eq:network_budget}  
\end{align}
Building on existing research that shows deeper early exits typically yield higher inference accuracy~\cite{teerapittayanon2016branchynet, zeng2019boomerang, bacca_distributed}, we observe that smaller nodes should prioritize offloading requests to their parent nodes.\footnote{We illustrate this point by contradiction: suppose in the optimal solution of $\optimprob_1$, there exists a node with $\fraction^*_i > 0$ and $\transrate_i < \maxtransrate_i$. By decreasing $\fraction^*_i$ to $0$ or to a value that makes $\transrate_i=\maxtransrate_i$ and adjusting $f^*_j$ s.t. parent node $j$ serves these additional requests locally, we achieve another feasible solution to $\optimprob_1$ with a smaller loss.}
This allows us to simplify the optimization problem $\optimprob_1$ by restricting the search space to strategies that prioritize offloading, resulting in an equivalent optimization problem, $\optimprob_2$, which focuses on minimizing losses at early exits:
\begin{align}\label{eq:inference_aware_cen_sim}
        &\optimprob_2:&\phantom{} &\min_{\ww \in \mathcal{W}} \sum_{e\in\mathcal{E}} \Lambda_e ~\E_{z\sim \servedataset_e}   \brackets{ \ell^{(e)}(\ww, z)}, 
\end{align}
where $\Lambda_e \triangleq \sum_{i\in \setnodesexit} \parentheses{  \requestrate_i + \sum_{j \in \setchildnodes_i} \transrate_j  - \transrate_{i}}$ is the total serving rate of all nodes using exit $e$ at inference time, and
the data distribution of serving samples at early exit $e$ is $\servedataset_e$.\footnote{In $\optimprob_2$, the transfer rates are given by the recurrence relation $\transrate_i = 
\min\set{\maxtransrate_i, \requestrate_i + \sum_{j \in \setchildnodes_i} \transrate_j}$.} 

In $\optimprob_2$, the serving rates $\Lambda_e$ are constant, depending only on the arrival rates $\requestrate_i$ and the maximum
transfer rates $\maxtransrate_i$.
Before training begins, the cloud can collect this information from all nodes to compute the serving rates $\Lambda_e$.

\subsection{Federated Learning Algorithm Dissection}
\label{subsec:algorithm}
\begin{algorithm}
\caption{Federated Learning for Distributed EENs}
\begin{algorithmic}[1]
\STATE{\textbf{Input:} a randomized initial model $\ww_1$, total communication rounds $T$, local steps $J$, global learning rate $\eta_s$, local learning rates $\{\eta^{(t,j)}\}$ at round $t$ and local step $j$, sampling matrix $\boldsymbol{p}$, aggregation weights $\LLa$.}
\FOR{$t = 1$ to $T$}
    \STATE Server samples the set $\mathcal{N}^{(t)}$ of node/exit pairs w.r.t.~$\boldsymbol{p}$.
    \STATE Server broadcasts the model $\ww^{(t)}$ to all nodes in $\mathcal{N}^{(t)}$.
    \FORALL{$(i,e) \in \mathcal{N}^{(t)}$ in parallel}
        \STATE $\ww^{(t,0)}_{i,e} = \ww^{(t)}$
        \FOR{$j = 0$ to $J-1$}
            \STATE Node $i$ selects a random batch $\mathcal{B}_i$\STATE{$$\hspace{-0.6cm}\ww^{(t,j+1)}_{i,e} = \wtjce - \eta^{(t,j)} \frac{1}{\left| \mathcal{B}_i \right|} \sum_{z \in \mathcal{B}_i}\nabla \ell^{(e)}(\wtjce, z)$$}
        \ENDFOR
        \STATE Node $i$ sends $\ww^{(t,J)}_{i,e}$ to the server
    \ENDFOR
    \STATE The server updates its global model
    \STATE $$\ww^{(t+1)} = \Pi_\mathcal{W} \left(\ww^{(t)} + \eta_{s} \sum_{(i,e) \in \mathcal{N}^{(t)}}  \frac{\tilde{\Lambda}_e |S_i|}{|\Sep|} \frac{g^{(t,J)}_{i,e}}{p_{i,e}}\right),$$ \\ \hspace{1.2cm} \text{where } $g^{(t,J)}_{i,e} = (\ww^{(t,J)}_{i,e} - \ww^{(t)})$ \label{alg-line:aggregation}
\ENDFOR
\STATE{return $\ww^{(T)}$}
\end{algorithmic}
\label{alg:fed-cis}
\end{algorithm}

We propose a FL algorithm that enables network nodes to collaboratively train an EEN for Problem~$\optimprob_2$ using their local datasets. Each node $i$ is designed to hold all exits up to its largest exit $E_i$. Although node $i$ only uses exit $E_i$ for inference, it can still play a crucial role in training smaller exits, particularly when it owns a substantial amount of data.
At each communication round, the server follows a two-step sampling process: first, it samples a set of nodes to participate in training, as in traditional FL algorithms; then, it selects a specific early exit for each chosen node to train.
The probability that a node $i$ is selected to train a particular early exit $e$ is denoted by $p_{i,e}$, while $\boldsymbol{p}\in \reals^{N\times E}$ represents the overall probability matrix.
The set of nodes with a non-zero probability of training exit~$e$ is $\Ce \triangleq \{i\in\setnodes \mid p_{i,e}>0 \}$, and the set of all samples from nodes in $\Ce$ is $\Sep \triangleq \cup_{i \in \Ce} S_i$, where node $i$ holds samples $S_i$.

Our FL algorithm aims to minimize a proxy of the objective in~$\optimprob_2$, where the expected loss at each early exit $e$ is replaced by the empirical loss computed on the dataset $\Sep$. 
Rather than strictly matching the weight $\tilde \Lambda_e$ to the expected inference request rate $\Lambda_e$, we adopt a more flexible training strategy that allows them to differ.
This choice is supported by our theoretical results in Sec.~\ref{subsec:theory}. However, even without the analysis, it is evident that when exit $e$ has a high inference request rate $\Lambda_e$ but limited data $|\Sep|$, the empirical loss may be too noisy, making it preferable to set \mbox{$\tilde \Lambda_e \ll \Lambda_e$}.


Our algorithm (Alg.~\ref{alg:fed-cis}) works as follows: 
At each communication round $t$, the server samples nodes and their corresponding early exits based on the probability matrix $\boldsymbol{p}$ (Lines~2-3). 
The server then broadcasts the current global model to the sampled nodes (Line~4). 
Each node $i$ performs multiple steps of mini-batch gradient descent on the loss associated with its sampled early exit $e$, and returns the updated model to the server (Lines~5-10).
The server aggregates these updates by computing a weighted sum of the pseudo-gradients from each node-exit pair $(i,e)$ (Lines~10-12). Each pair's weight is determined by three key factors: (i) the importance $\tilde \Lambda_e$ assigned to exit $e$; (ii) the proportion of the dataset that node $i$ used to train relative to the total dataset used to train exit $e$ ($\frac{|S_i|}{|\Sep|}$); and (iii) the inverse of the probability that node $i$ was selected to train exit $e$ ($\frac{1}{p_{i,e}}$). 

\subsection{Theoretical Results}\label{subsec:theory}
Our analytical results assume that the serving distribution of every exit $e$ is the same, i.e.,~$\hat{\mathcal{D}}_e=\mathcal{D}, \forall e$. 
Let $\ww^{(T)}$ be the output of Alg.~\ref{alg:fed-cis}, $F_{\mathcal{D}, \LL}(\ww^{(T)})$ be the corresponding expected loss in Eq.~\eqref{eq:inference_aware_cen_sim} that we aim to minimize, and $\Fu^\star$ be its minimum value.
In this section, we provide an upper-bound for the difference between  $F_{\mathcal{D}, \LL}(\ww^{(T)})$ and $\Fu^\star$.
More precisely, we investigate the \emph{true error} of the algorithm:
\begin{equation}
    \label{eq:true_error}
   \epsilon_{\text{true}} \triangleq \mathbb{E}_{S, A_{\LLa}}\Big[\Fu\Big(\ww^{(T)}\Big)\Big] - \Fu^\star, 
\end{equation}
where $A_{\LLa}$ is our algorithm and $S$ is the union of the nodes' datasets drawn from $\mathcal{D}$. We first list the assumptions needed for our results, denoting node $i$'s empirical loss on early exit $e$ of model~$\ww$ as $\Fce(\ww)$, i.e.,~$\Fce(\ww)\triangleq\frac{1}{|S_i|} \sum_{z \in S_i} \ell^{(e)}(\ww,z)$. We can see from our aggregation rule that Alg.~\ref{alg:fed-cis} is minimizing 
$\FSb(\ww) \triangleq \sum_{e\in E} \tilde{\Lambda}_e \sum_{i \in \Ce} \frac{|S_i|}{\sum_{i \in \Ce}  |S_i|} \Fce(\ww)$. Let $\ww_{i,e}^{\star}$, $\ww^{\star}_{\LLa}$, and $\ww_\mathcal{D}^{\star}$ be the minimizers of~$\Fce$, $\FSb$, and $\Fu$, respectively.

\begin{assumption}
    \label{asm:bounded_loss}
    (Bounded loss)
    The loss function is bounded, i.e., $\forall \ww \in \mathcal{W} 
    \text{ and } z\in\mathcal{Z},~\ell(\ww, z) \in [0, M]$.
\end{assumption}

\begin{assumption}\label{asm:hypothesis_class}
The hypothesis space $\mathcal{W}\subset \mathbb R^d$ is convex and compact with diameter $\diam(\mathcal{W})$, and contains the minimizers $\bm{w}_{i,e}^{\star}$, $\bm{w}_{\LLa}^{\star}$ and $\bm{w}_\mathcal{D}^{\star}$ in its interior.
\end{assumption}

\begin{assumption} \label{asm:smooth}
	$\{ \Fce \}_{(i,e) \in \mathcal{N}\times\mathcal{E}}$ are $L$-smooth:
	for all $\bm{v}$ and $\ww$ in $\mathcal{W}$, $\norm{\nabla \Fce(\vv) - \nabla \Fce(\ww)}_2 \leq L \norm{\bm{v} - \ww}_2$.
\end{assumption}

\begin{assumption} \label{asm:strong_cvx}
	$\{ \Fce \}_{(i,e) \in \mathcal{N}\times\mathcal{E}}$ are $\mu$-strongly convex:
	for all $\bm{v}$ and $\ww$ in $\mathcal{W}$, $\Fce(\bm{v})  \geq \Fce(\ww) + \scalar{\nabla \Fce(\ww)}{\bm{v} - \ww} + \frac{\mu }{2} \norm{\bm{v} - \ww}_2^2$.
\end{assumption}

\begin{assumption} \label{asm:sgd_var}
	Let $\mathcal{B}_i$ be a random batch sampled from the $i$-th node's local data uniformly at random.
	The variance of stochastic gradients in each node is bounded: $\E \norm{\nabla \Fce(\ww, \mathcal{B}_i) - \nabla \Fce(\ww)}^2 \le \sigma_{i,e}^2$ for all $\ww$ in $\mathcal{W}$ and $(i,e) \in \mathcal{N}\times\mathcal{E}$.
\end{assumption}

Assumption~\ref{asm:bounded_loss} is standard in statistical learning theory~(e.g.,~\citet{MohriRostamizadehTalwalkar18,shalev_shai}), while Assumptions~\ref{asm:hypothesis_class}--\ref{asm:sgd_var} are standard in the analysis of federated optimization algorithms (e.g.,~\citet{wangFieldGuideFederated2021,liConvergenceFedAvgNonIID2023,rodio23ton}).
We observe that Assumptions~\ref{asm:hypothesis_class}, \ref{asm:smooth}, and~\ref{asm:sgd_var} jointly imply that the stochastic gradients are bounded. We denote this bound by $G$, i.e., $\E \norm{\nabla \Fce(\ww, \mathcal{B}_i) }^2  \le G^2$ for $\ww \in \mathcal{W}$ and $(i,e) \in \mathcal{N}\times\mathcal{E}$.

Theorem~\ref{thm:convergence} provides an upper bound on the true error of our algorithm in terms of the sum of three components: a generalization error, a bias error (due to the mismatch between $\Fb$ and $\Fu$), and an optimization error. The proof is provided in the Technical Appendix.
\begin{theorem}
\label{thm:convergence}
Under Assumptions~\ref{asm:bounded_loss}--\ref{asm:sgd_var}, the true error of the output $\ww^{(T)}$ of Alg.~\ref{alg:fed-cis} with learning rate $\eta^{(t,j)}= \frac{2}{\mu (\gamma + (t-1)J+j+1)}$ and $\gamma\triangleq\max\{8 \kappa, J\}-1$ can be bounded as follows:
\begin{align}
\epsilon_{\text{true}} & \le \underbrace{\bigO \left(\sum_{e=1}^E \tilde{\Lambda}_e \sqrt{\frac{\Pd(H_e)}{{|\Sep|}}} \right)}_{\epsilon_{\text{gen}}}
 + \underbrace{\bigO\left(\tv(\LLa,\LL)\right)}_{\epsilon_{\text{bias}}}\nonumber\\
& + \underbrace{\bigO\left(\frac{B(\LLa, \bm{p},\bm{\sigma}, \{|S_{i}|\}_{(i,e) \in \mathcal{N}\times\mathcal{E}})}{J\times T}\right)}_{\epsilon_{\text{opt}}}, 
\end{align}
where $\kappa \triangleq \frac{L}{\mu}$, $Pdim(H_e)$ represents the pseudo-dimension of the class of models for exit $e$, $\tv$ is the total variation distance, $\LL = (\Lambda_1,\dots,\Lambda_E)$, $\LLa = (\tilde\Lambda_1,\dots,\tilde\Lambda_E)$, and 
the expression of $B(\cdot)$ is provided in the Technical Appendix.
\end{theorem}

\subsection{Configuration Rules}\label{subsec:proposed-method}
Theorem~\ref{thm:convergence}
shows that the choice of aggregation weights~$\LLa$ in Alg.~\ref{alg:fed-cis} (Line~\ref{alg-line:aggregation}) affects all three error components: generalization error, optimization error, and bias error---each minimized by a different choice of $\LLa$. 

The bias error $\epsilon_{\text{bias}}$ is dominant when each exit $e$ is trained on a large dataset $\Sep$ (making $\epsilon_{\text{gen}}$ small) and 
the number of communication rounds~$T$ is high (making $\epsilon_{\text{opt}}$ small). In such settings, the optimal strategy sets the aggregation weights $\LLa$ equal to the expected serving rates $\LL$. We refer to this configuration rule as ``Serving Rate'', which effectively eliminates the bias error, as $\tv(\LLa,\LL)=0$. However, optimization and generalization errors can also play a significant role. In these cases, deviating $\LLa$ from $\LL$ may reduce these errors, though it introduces a non-zero bias error.

The optimization error $\epsilon_{\text{opt}}$ is strongly influenced by the gradient variance $\sigma_{i,e}^2$ at each early exit $e$, as shown by the $B(\cdot)$ term, whose complete expression can be found in the Technical Appendix. Empirical evidence shows that gradient variance is significantly higher at the initial exits compared to the later ones, making the optimization error especially sensitive to the stochastic gradients produced at these early stages.\footnote{Experimental evidence is provided in the Technical Appendix.} To reduce $\epsilon_{\text{opt}}$, earlier exits with higher variance should be assigned lower aggregation weights $\tilde{\Lambda}_e < {\Lambda}_e$ to lessen their impact during training. In scenarios where the optimization error dominates, it follows that minimizing $\epsilon_{\text{opt}}$ involves setting the aggregation weights $\tilde \Lambda_e$ inversely proportional to the gradient variance $\sigma_{i,e}^2$. We observe that this approach alters the weights in the same direction as the ``FLOPS Prop'' strategy (described in Sec~\ref{sec:ee-background}), which also assigns larger weights to more powerful models.

The generalization error, on the other hand, is affected by the ratio $\Pd(H_e)/|\Sep|$. In practice, $\Pd(H_e)$ acts as a proxy for the complexity of the model at exit $e$. It follows that exits with a larger model (i.e., larger $\Pd(H_e)$), and smaller dataset (i.e., smaller $|\Sep|$), contribute more to this error component, and thus reducing the aggregation weights associated to these exits minimizes $\epsilon_{\text{gen}}$. In extreme cases where the generalization error is dominant, the optimal strategy requires setting aggregation weights to zero for all exits except the one with the lowest complexity ratio $\Pd(H_e)/|\Sep|$. The probabilities $p_{i,e}$ can also play a role in further reducing the generalization error, whereby powerful nodes periodically train exit~$e$, practically increasing the sample size $\Sep$ and leading to a reduced $\epsilon_{\text{gen}}$.

In many realistic scenarios, it is likely that no single error component is dominant, and one might consider configuring our FL algorithm by minimizing the entire bound in Theorem~\ref{thm:convergence}. However, this approach is often impractical due to the complexities involved in estimating theoretical parameters, such as the Lipschitz constant $L$ and the strong convexity constant $\mu$. To address this issue, our experimental findings suggest that a hybrid strategy, which balances the reduction of both bias and optimization errors, offers robust performance across many settings. For the remainder of this paper, we refer to this heuristic approach as ``Balanced Adj'', where the abbreviation ``Adj'' stands for Adjustment.

\section{Experiments}\label{sec:experiments}
In this section, we present experimental results that validate our theoretical analysis in Sec.~\ref{subsec:theory} and  highlight the versatility of our algorithm across various CIS serving rate settings.

\subsection{Training Details}
We conduct experiments on the CIFAR10 and CIFAR100 datasets, employing the ResNet-18 model architecture~\cite{he2016deep}. Both datasets and model are widely used to benchmark FL algorithms in the presence of device heterogeneity and EENs~\cite{li2019improved,DBLP:conf/aaai/HuDHB19,kaya2019shallow,diao2020heterofl,horvath2021fjord,ilhan2023scalefl}. 
We insert early exits after the 2nd and 5th residual blocks for CIFAR10 and after the 5th and 7th residual blocks for CIFAR100.
For reproducibility, all dataset details, training infrastructure, and hyperparameters are provided in the Technical Appendix.

\subsection{Evaluation Methodology}

\begin{table*}[htbp]
\centering
\caption{Experimental results for a variety of CIS serving rates on the CIFAR10 and CIFAR100 datasets using an equal data partition across the network layers. All reported accuracy values are the mean value over three independent random seeds.}
\begin{tabular}{
  l  
  l  
  c  
  c  
  c 
  c  
  c  
  c  
  c 
}
\toprule
\multicolumn{2}{c}{} & \multicolumn{7}{c}{CIS Serving Rate Setting} \\
\cmidrule(lr){3-9}
Dataset & Strategy & {80-15-5} & {60-30-10} & {45-35-20} & {33-33-33} & {20-35-45} & {10-30-60} & {5-15-80} \\
\midrule
\multirow{3}{*}{CIFAR10} 
  & Equal Weight  & 49.9 \scriptsize{$\pm$ 1.2} & 60.6 \scriptsize{$\pm$ 0.9} & 68.9 \scriptsize{$\pm$ 0.6} & \textbf{74.9} \scriptsize{$\pm$ 0.2} & 80.4 \scriptsize{$\pm$ 0.2} & 83.8 \scriptsize{$\pm$ 0.3} & 85.1 \scriptsize{$\pm$ 0.4} \\
  & FLOPS Prop    & 32.1 \scriptsize{$\pm$ 3.7} & 47.3 \scriptsize{$\pm$ 2.9} & 58.5 \scriptsize{$\pm$ 2.3} & 67.3 \scriptsize{$\pm$ 1.7} & 76.4 \scriptsize{$\pm$ 1.1} & 83.2 \scriptsize{$\pm$ 0.7} & 86.3 \scriptsize{$\pm$ 0.6} \\
  & Serving Rate (ours)  & \textbf{54.2} \scriptsize{$\pm$ 2.4} & \textbf{62.6} \scriptsize{$\pm$ 1.5} & \textbf{69.2} \scriptsize{$\pm$ 1.3} & \textbf{74.9} \scriptsize{$\pm$ 0.2} & \textbf{80.9} \scriptsize{$\pm$ 0.2} & 84.6 \scriptsize{$\pm$ 0.5} & 86.7 \scriptsize{$\pm$ 0.5} \\
  & Balanced Adj (ours) & 53.4 \scriptsize{$\pm$ 2.2} & 61.1 \scriptsize{$\pm$ 1.2} & \textbf{69.2} \scriptsize{$\pm$ 0.7} & \textbf{74.9} \scriptsize{$\pm$ 0.3} & 80.4 \scriptsize{$\pm$ 0.5} & \textbf{85.0} \scriptsize{$\pm$ 0.2} & \textbf{87.6} \scriptsize{$\pm$ 0.5} \\
\addlinespace  
\hline
\addlinespace  
\multirow{3}{*}{CIFAR100} 
  & Equal Weight  & 39.8 \scriptsize{$\pm$ 1.2} & 45.9 \scriptsize{$\pm$ 0.9} & 51.0 \scriptsize{$\pm$ 0.6} & \textbf{55.2} \scriptsize{$\pm$ 0.3} & \textbf{58.3} \scriptsize{$\pm$ 0.2} & 60.3 \scriptsize{$\pm$ 0.2} & 61.1 \scriptsize{$\pm$ 0.3} \\
  & FLOPS Prop    & 30.4 \scriptsize{$\pm$ 0.7} & 40.0 \scriptsize{$\pm$ 0.6} & 48.3 \scriptsize{$\pm$ 0.4} & 53.3 \scriptsize{$\pm$ 0.0} & 58.0 \scriptsize{$\pm$ 0.1} & \textbf{60.9} \scriptsize{$\pm$ 0.1} & \textbf{62.1} \scriptsize{$\pm$ 0.1} \\
  & Serving Rate (ours)  & 45.0 \scriptsize{$\pm$ 0.7} & \textbf{50.2} \scriptsize{$\pm$ 0.7} & \textbf{53.0} \scriptsize{$\pm$ 0.8} & \textbf{55.2} \scriptsize{$\pm$ 0.3} & 53.2 \scriptsize{$\pm$ 1.0} & 56.2 \scriptsize{$\pm$ 0.1} & 57.9 \scriptsize{$\pm$ 0.5} \\
  & Balanced Adj (ours)  & \textbf{46.6} \scriptsize{$\pm$ 1.2} & 49.5 \scriptsize{$\pm$ 0.8} & 52.1 \scriptsize{$\pm$ 0.6} & 54.8 \scriptsize{$\pm$ 0.3} & 57.4 \scriptsize{$\pm$ 0.2} & 59.3 \scriptsize{$\pm$ 0.8} & 60.7 \scriptsize{$\pm$ 0.2} \\ 
\bottomrule
\end{tabular}
\label{tab:equal-weight-results}
\end{table*}

\paragraph{Baselines.}
Our work represents the first attempt to develop a FL training algorithm for use within a CIS.
Due to the lack of established baselines for direct comparison, we compare our approach to SOTA algorithms proposed to train traditional EENs, focusing on those that have a straightforward application to FL and CIS settings (see Sec.~\ref{sec:ee-background} for a comprehensive description of these methods). 
The two strategies in this category are: (i) ``Equal Weight,'' which assigns equal weight to all early exits~\cite{teerapittayanon2017distributed,DBLP:conf/iclr/HuangCLWMW18}, and (ii) ``FLOPS Prop,'' which weights the exits according to their FLOPS~\cite{kaya2019shallow}. While other centralized training methods, such as those proposed by~\citet{DBLP:conf/aaai/HuDHB19,li2019improved}, could potentially be adapted for our purposes, their extension is less straightforward and would require extra computation by the nodes.  We also implement the (iii) ``Serving Rate'' and (iv) ``Balanced Adj'' strategies, both directly derived from our analysis in Sec.~\ref{subsec:theory}.
The code for our experimental framework is in the Supplementary Material.



\paragraph{CIS Topology.} 
We utilize a hierarchical network topology as defined in Sec.~\ref{subsec:formulation} and considered in related works~\cite{teerapittayanon2017distributed,ren2023survey} with seven nodes: four in the first layer, two in the second, and one in the third, each holding an increasing portion of the shared model according to their network layer.\footnote{We conducted additional experiments with a larger network consisting of 17 total nodes, which confirmed the consistency of our results. Due to computational constraints, this larger network was not used for all experiments. Detailed results are in the Technical Appendix.}
In Sec.~\ref{subsec:results}, we present results for two data partition settings: (a) ``equal data partition," where data is evenly distributed across all network layers, and (b) ``highly biased data partition," where data is heavily concentrated on the most powerful devices. Additional results for (c) ``biased data partition" are available in the Technical Appendix.

\paragraph{Serving Rates.} 
We assume that all inference requests initially arrive at the leaf nodes ($\requestrate_i=0, \forall i\in \setnodes \setminus \mathcal{L}$). During inference, each node $i$ assesses the confidence score of the incoming requests, serving the simplest ones based on its serving rate $\serverate_{i}$ and forwarding the remaining, more complex requests according to its transfer rate $\transrate_{i}$. 
We evaluate a wide range of serving rates $\LL$, including scenarios where (i)~the least powerful nodes serve most of the requests; (ii)~request rates are evenly distributed across all layers; and (iii) the most powerful nodes serve most of the requests. To denote these serving rates, we use the notation x-y-z, where x, y, and z represent the percentage of inference requests served by nodes using Exits 1, 2, and 3, respectively.

\subsection{Experimental Results}\label{subsec:results}
Table~\ref{tab:equal-weight-results} presents our results on the CIFAR10 and CIFAR100 datasets under the ``equal data partition'' setting. On CIFAR10, our ``Serving Rate'' and ``Balanced Adj'' strategies consistently outperform the ``Equal Weight'' and ``FLOPS Prop'' methods across all CIS serving rate configurations, especially in scenarios where the smallest models handle most of the inference requests, such as in the 80-15-5, 60-30-10, and 45-35-20 settings. In these cases, both ``Equal Weight'' and ``FLOPS Prop'' perform poorly, as they fail to account for the actual distribution of serving rates. 
Specifically, in the 80-15-5 setting, ``Serving Rate'' outperforms ``Equal Weight'' by 4.3 percentage points (p.p.) and ``FLOPS Prop'' by 22.1 p.p.,  while in the 5-15-80 setting, ``Balanced Adj'' surpasses them by 2.5 p.p. and 1.3 p.p., respectively.

To better understand these results, we analyze how different training strategies affect $\LLa$ and, in turn, the CIS test accuracy. First, setting $\LLa$ equal to the serving rate $\LL$ minimizes the bias error $\epsilon_{\text{bias}}$, which is the objective of our ``Serving Rate'' strategy.
On the CIFAR10 task, with $T=100$ communication rounds and a sufficiently large dataset, $\epsilon_{\text{bias}}$ dominates, allowing ``Serving Rate'' to empirically minimize this term and perform well across various serving rate settings. In contrast, ``Equal Weight'' assigns equal weights to all exits, which can significantly increase $\epsilon_{\text{bias}}$ as serving rates become more uneven, likely leading to poor performance in scenarios with extreme serving rate imbalances.

On the CIFAR100 dataset, we observe performance trends similar to CIFAR10 across various device-biased CIS serving rate settings, including 80-15-5, 60-30-10, 45-35-20, and 33-33-33. However, when the largest models handle most of the requests, the ``FLOPS Prop'' baseline outperforms our ``Serving Rate'' strategy, likely due to the greater difficulty of the CIFAR100 task, which results in a larger optimization error. As noted in Sec.~\ref{subsec:proposed-method}, strategies like ``FLOPS Prop'' are expected to perform well in these scenarios, though its performance drops significantly when the inference load shifts to the first-layer nodes. This shift increases the bias error because $\LLa$ diverges from $\LL$, causing a significant drop in CIS accuracy. This is evident in the 80-15-5 configuration, where ``Serving Rate'' and ``Balanced Adj'' outperform ``FLOPS Prop''  by 14.6 and 16.2 p.p., respectively.


In the Technical Appendix, we present results from experiments on CIFAR10 and CIFAR100 using the alternative ``biased data partition'' scheme, where nodes with greater memory and computational capacity are allocated more data. On CIFAR10, ``Serving Rate'' and ``Balanced Adj'' again consistently outperform other baselines across all CIS serving rate settings, while on CIFAR100, ``Balanced Adj'' remains strong in all scenarios, especially when the first-layer nodes handle most inference requests.

Combining these findings with those from the ``equal data partition'' experiments, our results show that the ``Balanced Adj'' strategy either leads or closely matches the performance of the best methods across various CIS configurations. Overall, these experiments reinforce the core insights from Sec.~\ref{subsec:proposed-method}, highlighting the critical role of error decomposition in selecting aggregation weights $\LLa$. In particular, configuring $\LLa$ to minimize bias error $\epsilon_{\text{bias}}$ proves often beneficial. Additionally, incorporating adjustments to address the optimization error $\epsilon_{\text{opt}}$---as done by ``Balanced Adj''---helps ensure that the resulting FL algorithm is robust across a wide range of serving rates, including the 80-15-5 and 5-15-80 settings.

\begin{table}[htbp]
\centering
\caption{Results for the 80-15-5 serving rate setting
using the highly biased data partition, where the network layers hold 3.4\%, 19.9\%, and 76.7\% of the data, respectively.} 
\begin{tabular}{
  l  
  l  
  c  
  c  
}
\toprule
Dataset & Strategy & {$p=0$} & {$p=0.2$} \\
\midrule
\multirow{2}{*}{CIFAR10}  
  & Equal Weight  & 36.5 \scriptsize{$\pm$ 4.0} & - \\
  & Serving Rate (ours)  & 41.3 \scriptsize{$\pm$ 2.6} & \textbf{49.2} \scriptsize{$\pm$ 2.9} \\
\addlinespace
\hline
\addlinespace
\multirow{2}{*}{CIFAR100}  
  & Equal Weight  & 10.0 \scriptsize{$\pm$ 1.4} & - \\
  & Serving Rate (ours)  & 17.9 \scriptsize{$\pm$ 0.2} & \textbf{30.1} \scriptsize{$\pm$ 2.1} \\
\bottomrule
\end{tabular}
\label{tab:p-analysis}
\end{table}

\paragraph{Enabling Node Collaboration through Probabilities~$\boldsymbol{p}$.}
We conducted an ablation study to examine the impact of the hyperparameter $\boldsymbol{p}$, focusing on extreme scenarios where nodes with the smallest models and datasets serve the majority of inference requests. This scenario is especially relevant, as end devices typically have limited data storage compared to cloud servers or other more powerful nodes. Results for the most challenging configuration, the 80-15-5 serving rate setting, are presented in Table~\ref{tab:p-analysis} for both the CIFAR10 and CIFAR100 datasets, where $p_{i,e}=p$ if $e< E_i$ and $p_{i,E_i}=1-(E_i-1)p$. These experiments clearly show that increasing $\boldsymbol{p}$ significantly improves the overall inference accuracy by enabling stronger nodes to support weaker ones during training. Additional results on the impact of $\boldsymbol{p}$ can be found in the Technical Appendix.

\section{Conclusion}\label{sec:conclusion}
We are the first to design an inference-aware FL training algorithm for CISs, demonstrating that inference serving rates influence all components of training error. When using our inference-aware configuration rules, which consider the error decomposition into the training process, our algorithm provides a significant advantage, particularly when inference request rates are unevenly distributed across the network. Moreover, our rigorous theoretical results are applicable to all approaches that jointly train models sharing a subset of parameters, including early exit networks, ordered dropout, pruning, and other nested training methodologies. 


\section*{Acknowledgements}
This research was supported in part by ANRT in the framework of a CIFRE PhD (2021/0073) and by the Horizon Europe project dAIEDGE.

\bibliography{refs}

\newpage
\onecolumn
\appendix
\section{Generalization and Bias Error}
\setcounter{theorem}{0}
\begin{theorem}
\label{thm:convergence_app}
Under Assumptions~\ref{asm:bounded_loss}--\ref{asm:sgd_var}, the true error of the output $\ww^{(T)}$ of Algorithm~\ref{alg:fed-cis} with learning rate $\eta_{t,j}= \frac{2}{\mu (\gamma + (t-1)J+j+1)}$ and $\gamma\triangleq\max\{8 \kappa, J\}-1$ can be bounded as follows:
\begin{align}
\epsilon_{\text{true}} & \le \underbrace{\bigO \left(\sum_{e=1}^E \tilde{\Lambda}_e \sqrt{\frac{\Pd(H_e)}{{|\Sep|}}} \right)}_{\epsilon_{\text{gen}}}
 + \underbrace{\bigO\left(\tv(\LLa,\LL)\right)}_{\epsilon_{\text{bias}}}
+ \underbrace{\bigO\left(\frac{B(\LLa, \bm{p},\bm{\sigma}, \{|S_{i}|\}_{(i,e) \in \mathcal{N}\times\mathcal{E}})}{J\times T}\right)}_{\epsilon_{\text{opt}}}, 
\end{align}
where $\kappa \triangleq \frac{L}{\mu}$, $Pdim(H_e)$ represents the pseudo-dimension of the class of models for exit $e$, $\tv$ is the total variation distance, $\LL = (\Lambda_1,\dots,\Lambda_E)$, $\LLa = (\tilde\Lambda_1,\dots,\tilde\Lambda_E)$, and 
the expression of $B(\cdot)$ is provided in Theorem~\ref{thm:opt_error}.
\end{theorem}
\begin{proof}
We start upperbounding the true error by three terms: a generalization error, a bias error (due to the mismatch between $\Fb$ and $\Fu$), and an optimization error:
\begin{align}
        \epsilon_{\text{true}} 
        & \leq 
        2 \E_{S}\left[\sup_{\ww} \left\lvert \FSb(\ww) - \Fu(\ww) \right\rvert \right] + \mathbb{E}_{S, A_{\LLa}}\left[ \FSb(\ww^{(T)}) - \FSb^\star\right]\\
        & \leq 
        2 \underbrace{\E_{S}\left[\sup_{\ww} \left\lvert \FSb(\ww) - \Fb(\ww) \right\rvert \right]}_{\epsilon_{\text{gen}}} 
        + 2 \underbrace{\E_{S}\left[\sup_{\ww} \left\lvert \Fb(\ww) - \Fu(\ww) \right\rvert \right]}_{\epsilon_{\text{bias}}}
        + \underbrace{\mathbb{E}_{S, A_{\LLa}}\left[ \FSb(\ww^{(T)}) - \FSb^\star\right]}_{\epsilon_{\text{opt}}},
        \label{error-decomp}
\end{align}
where the first inequality is quite standard (e.g.,~\cite[Eq.~9]{marfoq23datastreams}.
We obtain the final result by bounding each term.

For the generalization term, let $\Fe(\ww) \triangleq \E_{z \sim \mathcal{D}} [\ell^{(e)}(\ww,z)]$, we observe that
\begin{align}
        \epsilon_{\text{gen}} 
        & \le \sum_{e=1}^E \tilde{\Lambda}_e 
        \E_{S} \left[ \sup_{\ww} \left\lvert \left( \sum_{i \in \Ce} \frac{|S_i|}{|\Sep|}  \Fce(\ww) \right) - \Fe(\ww) \right\rvert \right]\nonumber\\
        & = \sum_{e=1}^E \tilde{\Lambda}_e 
        \E_{S} \left[ \sup_{\ww} \left\lvert \FSe(\ww) - \Fe(\ww) \right\rvert \right].
\end{align}
We can then bound the (expected) representativity $\E_{S} \left[ \sup_{\ww} \left\lvert \FSe(\ww) - \Fe(\ww) \right\rvert \right]$ for each exit $e$. Our task is not necessarily a binary classification task, but its representativity can be bounded by the representativity of an opportune classification task with the $0$-$1$~loss and  set of classifiers $H^\prime_e=\{h_{\ww,t}(z), \ww \in \mathcal{W}, t \in \mathbb R^+\}$, where $h_{\ww,t}(z) = \mathds{1}_{\ell^{(e)}(\ww,z)>t} $~\cite[Sec.~11.2.3]{MohriRostamizadehTalwalkar18}. In particular, let $\mathcal R_S(H)$ denote the Rademacher complexity of class $H$ on dataset S and let $\FSe^\prime(\ww,t)$ and $\Fe^\prime(\ww,t)$ denote the empirical loss and the expected loss for such classification problem, respectively. The analysis is then quite standard:
\begin{align}
\E_{S} & \left[ \sup_{\ww} \left\lvert \FSe(\ww) - \Fe(\ww) \right\rvert \right] \nonumber\\
& \le M \E_{S} \left[ \sup_{\ww} \left\lvert \FSe^\prime(\ww,t) - \Fe^\prime(\ww,t) \right\rvert \right]\\
& \le 2 M \mathcal \E_{\Sep}\left[\mathcal R_{\Sep}(H^\prime_e) \right]\\
& \le M C \sqrt{\frac{\VC(H^\prime_e)}{|\Sep|}}\\
& = M C \sqrt{\frac{\Pd(H_e)}{|\Sep|}}\label{eq:per-exit-gen-error}.
\end{align}
For a proof of the three inequalities the reader can refer to~\cite[Thm.~11.8]{MohriRostamizadehTalwalkar18}, \cite[Lm.~26.2]{shalev_shai}, \cite[Sec.~5]{BousquetBL03}, respectively (the constant $C$ can be selected to be $320$~\cite[Cor.~6.4]{livesay}). The final equality follows from the definition of pseudo-dimension.

For the bias term $\epsilon_{\text{bias}}$, it is sufficient to observe that
\begin{align}
    \label{eq:bias_bound}
        \epsilon_{\text{bias}}
        & \leq 
        \E_{S}\left[\sup_{\ww} \left\lvert \sum_{e =1}^E \left(\tilde{\Lambda}_e - \Lambda_e\right) \Fe(\ww)\right\rvert \right]\\
        & \leq 2 M \tv(\LLa,\LL). 
\end{align}
\end{proof}

\section{Optimization Error}
Our proof is similar to the proofs in~\cite{salehiFederatedLearningUnreliable2021,rodio23wpmc}. We adapt our notation to follow more closely that in those papers.

Let us consider the node update rule and the server aggregation rule in our algorithm:
\begin{align}
    \ww^{(t,j+1)}_{i,e} & = \wtjce - \eta_{t,j} \frac{1}{\left| \mathcal{B}^{(t,j)}_{i,e} \right|} \sum_{z \in \mathcal{B}^{(t,j)}_{i,e}} \nabla \ell_{e} (\wtjce, z), \textrm{  for } j=0,\dots, J-1\\
    \ww^{(t+1)} & = \ww^{(t)} + \eta_{s} \sum_{e \in E} \tilde{\Lambda}_e \sum_{i \in \mathcal N_{t,e}} \frac{|S_i|}{|\Sep|} \frac{1}{p_{i,e}} (\ww^{(t,J)}_{i,e} - \ww^{(t)}), \textrm{  for }t=1, \dots T .
\end{align}

We consider that a node corresponds to the pair $k \triangleq (i, e) \in \mathcal{K}$, where $\mathcal{K} \triangleq \mathcal{N} \times \mathcal{E}, i \in \mathcal{N}, \ e \in \mathcal{E}$, and we define $\alpha_k \triangleq \alpha_{i,e} \triangleq\eta_s \tilde{\Lambda}_e \frac{|S_i|}{|\Sep|}$, $\xi^{(t)}_k \triangleq \xi^{(t)}_{i,e} \triangleq  \mathds{1}_{i \in \mathcal N_{t,e}}$, and $\nabla F_k(\wtjce, \mathcal{B}^{(\tau)}_k)\triangleq \frac{1}{|\mathcal{B}^{(t,j)}_{i,e}|}\sum_{z \in \mathcal{B}^{(t,j)}_{i,e}} \nabla \ell_{e} (\wtjce, z)$.

Moreover, we count gradient steps at nodes and aggregation steps at the server using the same time sequence  $(\tau = J (t-1) + j)_{t=1, \dots, T, j=0, \dots, J-1}$. The set of values $\mathcal I^{(J)} =\{J t, t= 1, \dots, T\}$ corresponds to the aggregation steps. The equations above can then be rewritten as follows in terms of two new virtual sequences: 
\begin{align}
    \bm{v}^{(\tau+1)}_{k} &= \bm{w}^{(\tau)}_k - \eta_{\tau} \nabla F_k(\bm{w}^{(\tau)}_k, \mathcal{B}^{(\tau)}_k) \\
    \bm{w}^{(\tau+1)}_{k} &= 
    \begin{dcases*}
        \ww^{(1)}   & for $\tau+1 = 0$,\\
        \Pi_{\mathcal{W}}\left(\bm{w}^{(\tau+1-J)}_{k} + \sum_{k \in \mathcal{K}} \frac{\alpha_k \xi^{(\tau+1-J)}_{k}}{p_k} (\bm{v}^{(\tau+1)}_{k} - \bm{w}^{(\tau+1-J)}_{k})\right) & for $\tau+1 \in \mathcal{I}^{(J)}$,\\
        \bm{v}^{(\tau+1)}_{k} & otherwise.\\
    \end{dcases*}
\end{align}
$\bm{v}^{(J (t-1) + j)}_{k}$ coincides then with the local model $\wtjce$ and $\bm{w}_{k}^{(J (t-1))}$ coincides with the global model $\ww^{(t)}$.

We observe that, for $\tau+1  \in \mathcal{I}^{(J)}$, $\bm{w}^{(\tau+1)}_{k} = \bm{w}^{(\tau+1)}_{k'}$ for any $k$ and $k'$, and that, for $\tau+1 \not \in \mathcal{I}^{(J)}$, ${\bm{v}}^{(\tau+1)}_k = {\bm{w}}^{(\tau+1)}_k$.
Moreover, define the average sequences $\bar{\bm{v}}^{(\tau+1)} = \sum_{k \in \mathcal{K}}\alpha_k \bm{v}^{(\tau+1)}_{k}$ and $\bar{\bm{w}}^{(\tau+1)} = \sum_{k \in \mathcal{K}}\alpha_k \bm{w}^{(\tau+1)}_{k}$ and similarly the average gradients $\bm{g}^{(\tau)} = \sum_{k \in \mathcal{K}}\alpha_k \nabla F_k(\bm{w}^{(\tau)}_k, \mathcal B^{(t)}_k)$ and $\bar{\bm{g}}^{(\tau)}= \sum_{k \in \mathcal{K}}\alpha_k \nabla F_k(\bm{w}^{(\tau)}_k)$.  
We also define the sequence 
\begin{align}
\bar{\bm{w}}^{(\tau+1)\dagger} = 
    \begin{dcases*}
        \bm{w}^{(\tau+1-J)}_{k} + \sum_{k \in \mathcal{K}} \frac{\alpha_k \xi^{(\tau+1-J)}_{k}}{p_k} (\bm{v}^{(\tau+1)}_{k} - \bm{w}^{(\tau+1-J)}_{k}),& for $\tau+1 \in \mathcal{I}^{(J)}$\\
        \bar{\bm{w}}^{(\tau+1)},& otherwise.
    \end{dcases*}
\end{align}
We note that $\bar{\bm{w}}^{(\tau+1)}=\Pi_{\mathcal{W}}\left(\bar{\bm{w}}^{(\tau+1)\dagger}\right)$ for $\tau+1 \in \mathcal I^{(J)}$ and coincide otherwise.

We denote by $\mathcal B^{(\tau)} = (\mathcal B^{(\tau)}_k)_{k \in \mathcal{K} }$ and $\xi^{(\tau)} = (\xi^{(\tau)}_k)_{k \in \mathcal{K} }$, the set of batches and the set of indicator variables for node participation at instant $\tau$. The history of the system at time $\tau$ is made by the values of the random variables until that time and it can be defined by recursion as follows: 
$\mathcal H^{(1)} = \emptyset$, 
$\mathcal H^{(\tau+1)} = \{\xi^{(\tau+1)}, \mathcal{B}^{(\tau)}, \mathcal H^{(\tau)}\}$ if $\tau+1 \in \mathcal{I}^{(J)}$ and $\mathcal H^{(\tau+1)} = \{\mathcal{B}^{(\tau)}, \mathcal H^{(\tau)}\}$, otherwise.

We define $G_{i,e}\triangleq \sigma^2_{i,e} + (L \diam(\mathcal{W}))^2$ and observe that it bounds the second moment of the stochastic gradient at $(i,e)$:
\begin{align}
\E\left[\norm{\nabla \Fce(\ww, \mathcal B)}^ 2\right] & = \E\left[\norm{\nabla \Fce(\ww, \mathcal B) - \nabla \Fce(\ww)}^ 2 \right]+ \norm{\nabla \Fce(\ww)}^2\\
& \le \sigma_{i,e}^ 2 + L^2 \norm{\ww - \ww_{i,e}^*}^2\\
& \le \sigma_{i,e}^ 2 + L^2 \diam(\mathcal{W})^ 2 \\
& = G_{i,e},
\end{align}
where we have used Assumption~\ref{asm:hypothesis_class}.
We also define a uniform bound over all nodes and all exits: $G \triangleq \max_{(i,e) \in \mathcal{N} \times \mathcal{E}} G_{i,e}$.

Similarly to other works \cite{li2019convergence, li2020federated, wang2020tackling, wangFieldGuideFederated2021}, we introduce a metric to quantify the heterogeneity of nodes' local datasets, typically referred to as \emph{statistical heterogeneity}:
\begin{align}
    \Gamma &\triangleq \max_{(i,e) \in \mathcal{N}\times\mathcal{E}}  \Fce(\ww_{\LLa}^{\star}) - \Fce^{\star}.
    \label{eq:gamma}
\end{align}
Finally, we define $h(\tau) \triangleq\max\{ \tau' \in \mathcal{I}^{(J)} : \tau'\le \tau\}$. Then $h(\tau)$ indicates the time of the last server update before $\tau$.

The following lemma corresponds to \cite[Lemma~4]{salehiFederatedLearningUnreliable2021}.

\begin{lemma} 
\label{lem:1}
\begin{align}
    \E_{\mathcal{B}^{(\tau)}, \dots, \mathcal{B}^{(h(\tau))} \mid \mathcal{H}^{(h(\tau))}} \norm{\bar{\bm{v}}^{(\tau+1)} - \ww^{\star}_{\LLa}} 
    &\leq
    (1 - \eta_{\tau} \mu) \E_{\mathcal{B}^{(\tau-1)}, \dots, \mathcal{B}^{(h(\tau))} \mid \mathcal{H}^{(h(\tau))}} \norm{\bar{\bm{w}}^{(\tau)} - \ww^{\star}_{\LLa}} \notag \\
    &
    + \eta_{\tau}^2 \left( \sum_{k \in \mathcal{K}} \alpha_k^2 \sigma_k^2 + 6L \Gamma + 8 (J-1)^2 G^2 \right).
\end{align}
\end{lemma}
\begin{proof}
From~\cite[Lemma~1]{liConvergenceFedAvgNonIID2023}:
\begin{align}
    \E_{\mathcal{B}^{(\tau)} \mid \mathcal{H}^{(\tau)}} \norm{\bar{\bm{v}}^{(\tau+1)} - \ww^{\star}_{\LLa}} 
    &\leq
    (1 - \eta_{\tau} \mu) \norm{\bar{\bm{w}}^{(\tau)} - \ww^{\star}_{\LLa}} 
    + \eta_{\tau}^2 \E_{\mathcal{B}^{(\tau)} \mid \mathcal{H}^{(\tau)}} \norm{\bm{g}^{(\tau)} - \bar{\bm{g}}^{(\tau)}}^2 \notag \\
    &
    + \eta_{\tau}^2 6L \Gamma
    + 2 \sum_{k \in \mathcal{K}}\alpha_k \norm{\bm{w}^{(\tau)}_k - \bar{\bm{w}}^{(\tau)}}^2.
\end{align}
From~\cite[Lemma~2]{liConvergenceFedAvgNonIID2023}:
\begin{align}
    \E_{\mathcal{B}^{(\tau)} \mid \mathcal{H}^{(\tau)}} \norm{\bm{g}^{(\tau)} - \bar{\bm{g}}^{(\tau)}}^2
    &\leq 
    \E_{\mathcal{B}^{(\tau)} \mid \mathcal{H}^{(\tau)}} \norm{\sum_{k \in \mathcal{K}} \alpha_k \left( \nabla F_k(\bm{w}^{(t)}_{k}, \mathcal{B}^{(\tau)}_{k}) - \nabla F_k(\bm{w}^{(t)}_{k}) \right)}^2 \\
    &=
    \sum_{k \in \mathcal{K}} \alpha_k^2 \E_{\mathcal{B}^{(\tau)}_{k} \mid \mathcal{H}^{(t)}} \norm{\nabla F_k(\bm{w}^{(t)}_{k}, \mathcal{B}^{(\tau)}_{k}) - \nabla F_k(\bm{w}^{(t)}_{k})}^2 \\
    &\leq
    \sum_{k \in \mathcal{K}} \alpha_k^2 \sigma_k^2.
\end{align}

Combining the two inequalities above:
\begin{align}
    \E_{\mathcal{B}^{(\tau)} \mid \mathcal{H}^{(\tau)}} \norm{\bar{\bm{v}}^{(\tau+1)} - \ww^{\star}_{\LLa}} 
    &\leq
    (1 - \eta_{\tau} \mu) \norm{\bar{\bm{w}}^{(\tau)} - \ww^{\star}_{\LLa}} 
    + \eta_{\tau}^2 \sum_{k \in \mathcal{K}} \alpha_k^2 \sigma_k^2 
    + \eta_{\tau}^2 6L \Gamma
    + 2 \sum_{k \in \mathcal{K}}\alpha_k \norm{\bm{w}^{(\tau)}_k - \bar{\bm{w}}^{(\tau)}}^2.
\end{align}

By definition of $h(\tau)$, we observe that $0\le \tau- h(\tau) \le J-1$ and $\mathcal H^{(\tau)} = \{\mathcal B^{(\tau-1)}, B^{(\tau-2)}, \dots, B^{(h(\tau))}, \mathcal H^{(h(\tau))}\}$.

From~\cite[Lemma~3]{liConvergenceFedAvgNonIID2023}: \\
\begin{align}
    \sum_{k \in \mathcal{K}}\alpha_k &\E_{\mathcal{B}^{(\tau-1)}, \dots, \mathcal{B}^{(h(\tau))} \mid \mathcal{H}^{(h(\tau))}} \norm{\bm{w}^{(\tau)}_k - \bar{\bm{w}}^{(\tau)}}^2 \notag \\
    &=
    \sum_{k \in \mathcal{K}}\alpha_k \E_{\mathcal{B}^{(\tau-1)}, \dots, \mathcal{B}^{(h(\tau))} \mid \mathcal{H}^{(h(\tau))}} \norm{(\bm{w}^{(\tau)}_k - \bar{\bm{w}}^{(h(\tau))}) - (\bar{\bm{w}}^{(\tau)} - \bar{\bm{w}}^{(h(\tau))})}^2 \\
    &\leq
    \sum_{k \in \mathcal{K}}\alpha_k \E_{\mathcal{B}^{(\tau-1)}, \dots, \mathcal{B}^{(h(\tau))} \mid \mathcal{H}^{(h(\tau))}} \norm{\bm{w}^{(\tau)}_k - \bar{\bm{w}}^{(h(\tau))}}^2 \\
    &=
    \sum_{k \in \mathcal{K}}\alpha_k \E_{\mathcal{B}^{(\tau-1)}, \dots, \mathcal{B}^{(h(\tau))} \mid \mathcal{H}^{(h(\tau))}} \norm{\sum_{i=h(\tau)}^{t-1} \eta_i \nabla F_k(\bm{w}^{(i)}_{k}, \mathcal{B}^{(i)}_{k})}^2 \\
    &\leq
    \sum_{k \in \mathcal{K}}\alpha_k (\tau-h(\tau)) \E_{\mathcal{B}^{(\tau-1)}, \dots, \mathcal{B}^{(h(\tau))} \mid \mathcal{H}^{(h(\tau))}} \left[ \sum_{i=h(\tau)}^{t-1} \eta_i^2 \norm{\nabla F_k(\bm{w}^{(i)}_{k}, \mathcal{B}^{(i)}_{k})}^2 \right] \\
    &\leq
    \eta_{h(\tau)}^2 (t-h(\tau))^2 G^2 \vphantom{\sum_{k \in \mathcal{K}}} \\
    &\leq
    4 \eta_{\tau}^2 (J-1)^2 G^2. \vphantom{\sum_{k \in \mathcal{K}}}
\end{align}
By repeatedly computing expectations over the previous batch conditioned on the previous history and combining the inequalities above, we obtain:
\begin{align}
    \E_{\mathcal{B}^{(\tau)}, \dots, \mathcal{B}^{(h(\tau))} \mid \mathcal{H}^{(h(\tau))}} \norm{\bar{\bm{v}}^{(\tau+1)} - \ww^{\star}_{\LLa}} 
    &\leq
    (1 - \eta_{\tau} \mu) \E_{\mathcal{B}^{(\tau-1)}, \dots, \mathcal{B}^{(h(\tau))} \mid \mathcal{H}^{(h(\tau))}} \norm{\bar{\bm{w}}^{(\tau)} - \ww^{\star}_{\LLa}} \notag \\
    & + \eta_{\tau}^2 \left( \sum_{k \in \mathcal{K}} \alpha_k^2 \sigma_k^2 + 6L \Gamma + 8 (J-1)^2 G^2 \right).
\end{align}
\end{proof}

The following lemma corresponds to \cite[Lemma~2]{salehiFederatedLearningUnreliable2021}, but it needs to be adapted to take into account the projection.
\begin{lemma} 
\label{lem:2}
\begin{align}
    \E_{\xi^{(h(\tau))} \mid \mathcal{B}^{(\tau)}, 
    \dots, \mathcal{B}^{(h(\tau)+1)}, \mathcal{H}^{(h(\tau))}} [\bar{\bm{w}}^{(\tau+1)\dagger}] = \bar{\bm{v}}^{(\tau+1)}.
\end{align}
\end{lemma}
\begin{proof}
First, we observe that $\bar{\bm{w}}^{(\tau+1)\dagger} = \bar{\bm{w}}^{(\tau+1)} = \bar{\bm{v}}^{(\tau+1)}$ for $\tau+1 \not\in \mathcal I^{(J)}$. For $\tau+1 \in \mathcal I^{(J)}$, $h(\tau)=\tau+1-J$ and
\begin{align}
    &\E_{\xi^{(\tau+1-J)} \mid \mathcal{B}^{(\tau)}, 
    \dots, \mathcal{B}^{(\tau+1-J)}, \mathcal{H}^{(\tau+1-J)}} [\bar{\bm{w}}^{(\tau+1)\dagger}] = \notag \\
    &\qquad =
    \bar{\bm{w}}^{(\tau+1-J)} - \sum_{k \in \mathcal{K}}\frac{\alpha_k \E[\xi^{(\tau+1-J)}_{k}]}{p_k} \sum_{j=0}^{J-1} \eta_{\tau+1-J+j} \nabla F_k(\bm{w}^{(\tau+1-J+j)}_{k}, \mathcal{B}^{(\tau+1-J+j)}_{k}) \\
    &\qquad=
    \bar{\bm{w}}^{(\tau+1-J)} - \sum_{k \in \mathcal{K}}\alpha_k \sum_{j=0}^{J-1} \eta_{\tau+1-J+j} \nabla F_k(\bm{w}^{(\tau+1-J+j)}_{k}, \mathcal{B}^{(\tau+1-J+j)}_{k}) \\
    &\qquad=
    \bar{\bm{v}}^{(\tau+1)}.
\end{align}
\end{proof}

The following lemma corresponds to \cite[Lemma~3]{salehiFederatedLearningUnreliable2021}. We modify the proof to take into account the correlation in the participation of the fictitious nodes in $\mathcal{K}$. Indeed, each node $i$ selects a single exit to train and then the random variables $\{\xi^{(h(\tau))}\}_{e \in E}$ are (negatively) correlated. 
\begin{lemma} 
\label{lem:3}
\begin{align}
    \E_{\mathcal{B}^{(\tau)}, \dots, \mathcal{B}^{(h(\tau))}, \xi^{(h(\tau))} \mid \mathcal{H}^{(h(\tau))}} \norm{\bar{\bm{w}}^{(\tau+1)\dagger} - \bar{\bm{v}}^{(\tau+1)}}^2
    &\le 4 \eta_{\tau}^2 J^2 G^2 \sum_{i=1}^N\left( \sum_{e \in E_i} \frac{\alpha_{i,e}^2 }{p_{i,e}} 
     - 
        \left(\sum_{e \in E_i} \alpha_{i,e}\right)^2
     \right).
\end{align}
\end{lemma}
\begin{proof}
We have a tighter bound ($\alpha_k^2$ instead of $\alpha_k$), observing that $\Var(X) = \E[X-\E[X]]^2$.
Let $\bm{X}$ be a d-dimensional random variable, we define its variance as follows: $\VarVec(\bm{X})\triangleq \sum_{i=1}^d \Var(X_i)$. We also denote by $E_i$ the set of exits node $i$ may train, i.e., $E_i\triangleq \{e : p_{i,e} >0, e=1, \dots, E\}$. 

In order to keep the following calculations simpler to follow, we denote by $\bm{U}_{i,e}=\sum_{j=0}^{\tau-h(\tau)} \eta_{h(\tau)+j} \nabla \Fce(\bm{w}^{(h(\tau)+j)}_{i,e}, \mathcal{B}^{(h(\tau)+j)}_{i,e})$.

\begin{align}
    &\E_{\xi^{(h(\tau))} \mid \mathcal{B}^{(\tau)}, 
    \dots, \mathcal{B}^{(h(\tau))}, \mathcal{H}^{(h(\tau))}} \norm{\bar{\bm{w}}^{(\tau+1)\dagger} - \bar{\bm{v}}^{(\tau+1)}}^2\nonumber\\
    &\qquad=
    \VarVec_{\xi^{(h(\tau))} \mid \mathcal{B}^{(\tau)}, 
    \dots, \mathcal{B}^{(h(\tau))}, \mathcal{H}^{(h(\tau))}} 
    \left( \sum_{k \in \mathcal{K}}\frac{\alpha_k \xi^{(h(\tau))}_{k}}{p_k} \sum_{j=0}^{\tau-h(\tau)} \eta_{h(\tau)+j} \nabla F_k(\bm{w}^{(h(\tau)+j)}_{k}, \mathcal{B}^{(h(\tau)+j)}_{k}) \right) \\
    &\qquad=
    \VarVec_{\xi^{(h(\tau))} \mid \mathcal{B}^{(\tau)}, 
    \dots, \mathcal{B}^{(h(\tau))}, \mathcal{H}^{(h(\tau))}} 
    \left( \sum_{i=1}^N \sum_{e\in E_i}\frac{\alpha_{i,e} \xi^{(h(\tau))}_{i,e}}{p_{i,e}} \bm{U}_{i,e} \right) \\
    &\qquad=
    \sum_{i=1}^N \VarVec_{\xi^{(h(\tau))} \mid \mathcal{B}^{(\tau)}, 
    \dots, \mathcal{B}^{(h(\tau))}, \mathcal{H}^{(h(\tau))}} 
    \left(  \sum_{e\in E_i}\frac{\alpha_{i,e} \xi^{(h(\tau))}_{i,e}}{p_{i,e}} \bm{U}_{i,e} \right) \\
    &\qquad\le
    \sum_{i=1}^N \sum_{e\in E_i} \VarVec_{\xi^{(h(\tau))} \mid \mathcal{B}^{(\tau)}, 
    \dots, \mathcal{B}^{(h(\tau))}, \mathcal{H}^{(h(\tau))}} 
    \left(  \frac{\alpha_{i,e} \xi^{(h(\tau))}_{i,e}}{p_{i,e}} \bm{U}_{i,e} \right) \label{eq:corr}\\
    &\qquad= \sum_{i=1}^N \sum_{e\in E_i} \Var\left(\frac{\alpha_{i,e} \xi^{(h(\tau))}_{i,e}}{p_{i,e}} \right) \norm{\bm{U}_{i,e}}^2\\
    &\qquad= \sum_{i=1}^N \sum_{e\in E_i} \alpha_{i,e}^2 \frac{1-p_{i,e}}{p_{i,e}} \norm{\bm{U}_{i,e}}^2.
\end{align}
where \eqref{eq:corr} takes into account that $\xi^{(\tau+1-J)}_{i,e} \xi^{(\tau+1-J)}_{i,e'}=0$ for $e\neq e'$ because each node selects a single exit to train. 

Then, the expectation over the random batches is computed
\begin{align}
    &\E_{\mathcal{B}^{(\tau)}, \dots, \mathcal{B}^{(h(\tau))}, \xi^{(h(\tau))} \mid \mathcal{H}^{(h(\tau))}} \norm{\bar{\bm{w}}^{(\tau+1)\dagger} - \bar{\bm{v}}^{(\tau+1)}}^2 \nonumber\\
    &\qquad \le \sum_{i=1}^N \sum_{e\in E_i} \alpha_{i,e}^2 \frac{1-p_{i,e}}{p_{i,e}} 
    \E_{\mathcal{B}^{(\tau)}, \dots, \mathcal{B}^{(h(\tau))}  \mid \mathcal{H}^{(h(\tau))}}\left[\norm{\bm{U}_{i,e}}^2\right]\\
    &\qquad \le \sum_{i=1}^N \sum_{e\in E_i} \alpha_{i,e}^2 \frac{1-p_{i,e}}{p_{i,e}} 
    \E_{\mathcal{B}^{(\tau)}, \dots, \mathcal{B}^{(h(\tau))}  \mid \mathcal{H}^{(h(\tau))}}\left[
    \norm{\sum_{j=0}^{\tau-h(\tau)} \eta_{h(\tau)+j} \nabla \Fce(\bm{w}^{(h(\tau)+j)}_{i,e}, \mathcal{B}^{(h(\tau)+j)}_{i,e})}^2
    \right]\\
    &\qquad \le
     \eta_{h(\tau)+j}^2 J 
     \sum_{i=1}^N \sum_{e\in E_i} \alpha_{i,e}^2 \frac{1-p_{i,e}}{p_{i,e}}   \sum_{j=0}^{\tau-h(\tau)} \E_{\mathcal{B}^{(\tau)}, \dots, \mathcal{B}^{(h(\tau))}  \mid \mathcal{H}^{(h(\tau))}}\left[\norm{\nabla \Fce(\bm{w}^{(h(\tau)+j)}_{i,e}, \mathcal{B}^{(h(\tau)+j)}_{i,e})}^2 \right]\\
     &\qquad \le
     \eta_{h(\tau)+j}^2 J^2 
     \sum_{i=1}^N \sum_{e\in E_i} \alpha_{i,e}^2 \frac{1-p_{i,e}}{p_{i,e}}   G_{i,e}\\
    &\qquad \le
     4 \eta_{\tau}^2 J^2
    \sum_{i=1}^N \sum_{e\in E_i} \alpha_{i,e}^2 \frac{1-p_{i,e}}{p_{i,e}}   G_{i,e},\label{eq:eta}
\end{align}
where \eqref{eq:eta} uses $\eta_{h(\tau)+j} \leq  \eta_{\tau- J} \leq 2 \eta_\tau$.
\end{proof}

\begin{theorem} 
\label{thm:opt_error}
Under Assumptions~\ref{asm:hypothesis_class}--\ref{asm:sgd_var}, the optimization error of Algorithm~\ref{alg:fed-cis} with learning rate $\eta_{t,j}= \frac{2}{\mu (\gamma + (t-1)J+j+1)}$ and $\gamma\triangleq\max\{8 \kappa, J\}-1$ can be bounded as follows:
\begin{align}
    \E\left[ \FSb(\ww^{(T)}) \right] - \FSb^\star= \frac{\kappa}{\gamma + J T} \left(\frac{2B}{\mu} + \frac{\mu (\gamma+1)}{2} \E\left[\ww^{(1)} -\ww^{\star}_{\LLa}\right]\right),
\end{align}
where 
    \begin{align}
        & B \triangleq\sum_{(i,e) \in \mathcal{N}\times\mathcal{E}} \alpha_{i,e}^2 \sigma_{i,e}^2 + 6L \Gamma + 8 (J-1)^2 G^2 
    + 4 J^2     \sum_{i=1}^N \sum_{e\in E_i} \alpha_{i,e}^2 \frac{1-p_{i,e}}{p_{i,e}}   G_{i,e},\\
    & G_{i,e}  \triangleq \sigma^2_{i,e} + (L \diam(\mathcal{W}))^2,\\
    & G  \triangleq \max_{(i,e) \in \mathcal{N}\times\mathcal{E}} G_{i,e},\\
    & \alpha_{i,e} \triangleq \eta_s \tilde{\Lambda}_e \frac{|S_i|}{|\Sep|}.
    \end{align}
\end{theorem}
\begin{proof}
As we mention at the beginning of this appendix, we count gradient steps at nodes and aggregation steps at the server using the same time sequence  $(\tau = J (t-1) + j)_{t=1, \dots, T,j=0, \dots, J-1}$. The set of values $\mathcal I^{(J)} =\{J t, t= 1, \dots, T\}$ corresponds to the aggregation steps.

We have
\begin{align}
    \norm{\bar{\bm{w}}^{(\tau+1)} - \ww^{\star}_{\LLa}}^2
    & \le \norm{\bar{\bm{w}}^{(\tau+1)\dagger} - \ww^{\star}_{\LLa}}^2\\
    & = 
    \norm{\bar{\bm{w}}^{(\tau+1)\dagger} - \bar{\bm{v}}^{(\tau+1)} + \bar{\bm{v}}^{(\tau+1)} - \ww^{\star}_{\LLa}}^2 \\
    &=
    \norm{\bar{\bm{w}}^{(\tau+1)\dagger} - \bar{\bm{v}}^{(\tau+1)}}^2 + \norm{\bar{\bm{v}}^{(\tau+1)} - \ww^{\star}_{\LLa}}^2 + 2 \scalar{\bar{\bm{w}}^{(\tau+1)\dagger} - \bar{\bm{v}}^{(\tau+1)}}{\bar{\bm{v}}^{(\tau+1)} - \ww^{\star}_{\LLa}},
\end{align}
where the first inequality is trivially true for $\tau+1 \not\in \mathcal I^{(J)}$ because $\bar{\bm{w}}^{(\tau+1)}=\bar{\bm{w}}^{(\tau+1)\dagger}$, while for $\tau+1 \not\in \mathcal I^{(J)}$, it follows from Assumption~\ref{asm:hypothesis_class} and$\norm{\bar{\bm{w}}^{(\tau+1)} - \ww^{\star}_{\LLa}}^2= \norm{\Pi_{\mathcal{W}}(\bar{\bm{w}}^{(\tau+1)\dagger}) - \Pi_{\mathcal{W}}(\ww^{\star}_{\LLa})}^2 \le \norm{\bar{\bm{w}}^{(\tau+1)\dagger}- \ww^{\star}_{\LLa}}^2$.

We take expectation over nodes' participation
\begin{align}
   \E_{\xi^{(h(\tau))} \mid \mathcal{B}^{(\tau)}, 
    \dots, \mathcal{B}^{(h(\tau))}, \mathcal{H}^{(h(\tau))}} \norm{\bar{\bm{w}}^{(\tau+1)} - \ww^{\star}_{\LLa}}^2
    &\le
    \norm{\bar{\bm{v}}^{(\tau+1)} - \ww^{\star}_{\LLa}}^2 + \E_{\xi^{(h(\tau))} \mid \mathcal{B}^{(\tau)}, 
    \dots, \mathcal{B}^{(h(\tau))}, \mathcal{H}^{(h(\tau))}} \norm{\bar{\bm{w}}^{(\tau+1)\dagger} - \bar{\bm{v}}^{(\tau+1)}}^2 \\
    &\leq
    \norm{\bar{\bm{v}}^{(\tau+1)} - \ww^{\star}_{\LLa}}^2 
    + 4 \eta_{\tau}^2 J^2     \sum_{i=1}^N \sum_{e\in E_i} \alpha_{i,e}^2 \frac{1-p_{i,e}}{p_{i,e}}   G_{i,e},
\end{align}
where the equality derives from Lemma~\ref{lem:2} and the inequality from Lemma~\ref{lem:3}.
We take then expectation over the random batches
\begin{align}
    \E_{\mathcal{B}^{(\tau)}, \dots, \mathcal{B}^{(h(\tau))}, \xi^{(h(\tau))} \mid \mathcal{H}^{(h(\tau))}} & \norm{\bar{\bm{w}}^{(\tau+1)} - \ww^{\star}_{\LLa}}^2
    \nonumber \\
    &\leq
    \E_{\mathcal{B}^{(\tau)}, \dots, \mathcal{B}^{(h(\tau))} \mid \mathcal{H}^{(h(\tau))}} \norm{\bar{\bm{v}}^{(\tau+1)} - \ww^{\star}_{\LLa}}^2 
    + 4 \eta_{\tau}^2 J^2     \sum_{i=1}^N \sum_{e\in E_i} \alpha_{i,e}^2 \frac{1-p_{i,e}}{p_{i,e}}   G_{i,e}\\
    &\leq
    (1 - \eta_{\tau} \mu) \E_{\mathcal{B}^{(\tau-1)}, \dots, \mathcal{B}^{(h(\tau))} \mid \mathcal{H}^{(h(\tau))}} \norm{\bar{\bm{w}}^{(\tau)} - \ww^{\star}_{\LLa}} 
    + \eta_{\tau}^2 \left( \sum_{k \in \mathcal{K}} \alpha_k^2 \sigma_k^2 + 6L \Gamma + 8 (J-1)^2 G^2 \right)\nonumber\\
    & + 4 \eta_{\tau}^2 J^2     \sum_{i=1}^N \sum_{e\in E_i} \alpha_{i,e}^2 \frac{1-p_{i,e}}{p_{i,e}}   G_{i,e},
\end{align}
where the last inequality follows from Lemma~\ref{lem:1} observing that if $\tau+1 \in \mathcal I^{(J)}$, then $h(\tau)=\tau+1-J$.

Finally, we take total expectation
\begin{align}
    \E \norm{\bar{\bm{w}}^{(\tau+1)} - \ww^{\star}_{\LLa}}^2
    \leq
    (1 - \eta_{\tau} \mu) \E \norm{\bar{\bm{w}}^{(\tau)} - \ww^{\star}_{\LLa}} 
    + \eta_{\tau}^2 \left( \sum_{k \in \mathcal{K}} \alpha_k^2 \sigma_k^2 + 6L \Gamma + 8 (J-1)^2 G^2 
    + 4 J^2     \sum_{i=1}^N \sum_{e\in E_i} \alpha_{i,e}^2 \frac{1-p_{i,e}}{p_{i,e}}   G_{i,e}
    \right).
\end{align}
This leads to a recurrence relation of the form $\Delta^{(\tau+1)} \le (1- \eta_\tau \mu) \Delta^{(\tau)} + \eta_\tau^2 B,$
and the result is obtained following the same steps in the proof of~\cite[Thm.~1]{liConvergenceFedAvgNonIID2023}.
\end{proof}

\section{Gradient Variance Analysis}
As discussed in Section~\ref{subsec:proposed-method}, we observed empirical evidence showing that gradient variance is significantly higher at the initial exits compared to the later ones, making the optimization error especially sensitive to the stochastic gradients produced at these early stages. We conducted the following experiment: (1) Instantiate an Early Exit Network, e.g., a ResNet-18 with early exits after the 2nd and 5th residual blocks for CIFAR10 and after the 5th and 7th residual blocks for CIFAR100; (2) Iterate over the training data in mini-batches and calculate the gradient of the loss w.r.t. the weights at each exit; (3) Calculate the point-wise mean of each gradient over the mini-batches for each exit; (4) Take the mean of the gradient variance mean's to get a single value representing the average point-wise gradient variance per exit. We present below the empirical values from conducting this experiment:
\begin{table}[htbp]
\centering
\caption{Average Point-wise gradient variances per-exit.} 
\begin{tabular}{
  l 
  c
  c
  c
}
\toprule
Dataset & Exit 1 & Exit 2 & Exit 3 \\
\midrule
  CIFAR10 & 0.00374  & 0.00224 & 0.00101 \\
\addlinespace
\hline
\addlinespace
  CIFAR100 & 0.00216  & 0.00126 & 0.00101 \\
\bottomrule
\end{tabular}
\label{tab:gradient-variance}
\end{table}

\section{Training Details}
\paragraph{Datasets.} We use the CIFAR10 and CIFAR100 datasets, which are commonly used to benchmark FL algorithms and early exit networks~\cite{horvath2021fjord,diao2020heterofl,li2019improved,DBLP:conf/aaai/HuDHB19,kaya2019shallow,ilhan2023scalefl}. CIFAR10 and CIFAR100 each contain 60,000 total images composed of 32 x 32 colored pixels, with 10 and 100 classes, respectively. In our experiments, we use 45,000 images for training data, 5,000 images for validation data, and 10,000 images for test data. 

\paragraph{Model Architecture and Hyperparameters.} We conduct our experiments using a ResNet-18 model architecture~\cite{he2016deep}, which has been widely used to study early exit networks and device heterogeneity in FL~\cite{horvath2021fjord,diao2020heterofl,li2019improved,DBLP:conf/aaai/HuDHB19,kaya2019shallow,ilhan2023scalefl}. We insert early exits after the 2nd and 5th residual blocks for CIFAR10 and after the 5th and 7th residual blocks for CIFAR100. The training takes place for 100 outer epochs and the number of local epochs per node is scaled such that each node does the same number of gradient updates. We use mini-batch SGD with a starting learning rate of 0.1 and a cosine annealing schedule, a batch size of 128, weight decay of 5 × $10^{-4}$, and momentum of 0.9. These hyperparameter values were selected based on empirically observing convergence during training for several basic CIS configurations, e.g., equal data partition and 33-33-33 serving rate setting. The same values are used for all experiments, i.e., all training data partitions, CIS serving rate setting, and training strategy configurations. All presented results are the mean value over three random seeds: 9, 42, and 67.

\paragraph{Training Infrastructure.} 
We conducted our experiments on a computing node equipped with 3 x Nvidia A40 PCIe GPUs, each providing 10,752 CUDA cores, 336 tensor cores, and 48 GB of RAM. The node is powered by 2 x AMD EPYC 7282 processors running at 2.8 GHz, with 256 GB of system RAM. The operating system used was a Linux-based environment (e.g., Ubuntu 20.04), and the experiments were implemented using Python 3.8, CUDA 11.4, and cuDNN 8.2.

\section{Additional Experiments}
\begin{table*}[htbp]
\centering
\caption{Experimental results for a CIS with 17 nodes (12 in the first layer, 4 in the second, and 1 in the third) for several CIS serving rates on the CIFAR10 dataset using an equal data partition across the network layers. All reported accuracy values are the mean value over three independent random seeds. The performance of the strategies for each serving rate setting follows the exact same order as in Table 1, indicating that our experimental setup with seven nodes is adequate for capturing CIS dynamics observed at larger scales.}
\begin{tabular}{
  l  
  c  
  c  
}
\toprule
\multicolumn{1}{c}{} & \multicolumn{2}{c}{CIS Serving Rate Setting} \\
\cmidrule(lr){2-3}
Strategy & {60-30-10} & {10-30-60} \\
\midrule
Equal Weight  & 58.9 \scriptsize{$\pm$ 3.9} & 83.5 \scriptsize{$\pm$ 0.6} \\
FLOPS Prop & 44.6 \scriptsize{$\pm$ 1.5} & 82.4 \scriptsize{$\pm$ 0.5} \\
Serving Rate (ours) & \textbf{62.1} \scriptsize{$\pm$ 1.7} & 84.3 \scriptsize{$\pm$ 1.1} \\
Balanced Adj (ours) & 60.2 \scriptsize{$\pm$ 3.1} & \textbf{84.7} \scriptsize{$\pm$ 1.0} \\
\bottomrule
\end{tabular}
\label{tab:17-nodes-equal-partition-results}
\end{table*}

\begin{table*}[htbp]
\centering
\caption{Experimental results for a variety of CIS serving rates on the CIFAR10 and CIFAR100 datasets using the biased data partition, where the networks layers hold 14.3\%, 28.6\%, and 57.1\% of the data, respectively. All reported accuracy values are the mean value over three independent random seeds.}
\begin{tabular}{
  l  
  l  
  c  
  c  
  c 
  c  
  c  
  c  
  c 
}
\toprule
\multicolumn{2}{c}{} & \multicolumn{7}{c}{CIS Serving Rate Setting} \\
\cmidrule(lr){3-9}
Dataset & Strategy & {80-15-5} & {60-30-10} & {45-35-20} & {33-33-33} & {20-35-45} & {10-30-60} & {5-15-80} \\
\midrule
\multirow{3}{*}{CIFAR10} 
  & Equal Weight  & 47.4 \scriptsize{$\pm$ 3.6} & 58.9 \scriptsize{$\pm$ 3.5} & 67.8 \scriptsize{$\pm$ 2.8} & \textbf{74.9} \scriptsize{$\pm$ 2.2} & 80.9 \scriptsize{$\pm$ 1.4} & 85.0 \scriptsize{$\pm$ 0.7} & 86.7 \scriptsize{$\pm$ 0.2} \\
  & FLOPS Prop  & 31.5 \scriptsize{$\pm$ 3.2} & 47.2 \scriptsize{$\pm$ 2.5} & 59.0 \scriptsize{$\pm$ 1.8} & 68.4 \scriptsize{$\pm$ 1.4} & 78.1 \scriptsize{$\pm$ 0.8} & 85.4 \scriptsize{$\pm$ 0.4} & 88.8 \scriptsize{$\pm$ 0.4} \\
  & Serving Rate (ours)  & \textbf{53.1} \scriptsize{$\pm$ 2.3} & \textbf{60.6} \scriptsize{$\pm$ 0.7} & 66.4 \scriptsize{$\pm$ 1.1} & \textbf{74.9} \scriptsize{$\pm$ 2.2} & 81.7 \scriptsize{$\pm$ 1.7} & 87.1 \scriptsize{$\pm$ 0.7} & 89.2 \scriptsize{$\pm$ 0.6} \\
  & Balanced Adj (ours)  & 51.0 \scriptsize{$\pm$ 3.2} & 59.6 \scriptsize{$\pm$ 4.2} & \textbf{68.0} \scriptsize{$\pm$ 4.1} & 74.6 \scriptsize{$\pm$ 2.2} & \textbf{82.1} \scriptsize{$\pm$ 1.9} & \textbf{87.5} \scriptsize{$\pm$ 0.7} & \textbf{90.3} \scriptsize{$\pm$ 0.6} \\
\addlinespace  
\hline
\addlinespace  
\multirow{3}{*}{CIFAR100} 
  & Equal Weight  & 37.3 \scriptsize{$\pm$ 0.9} & 44.7 \scriptsize{$\pm$ 0.6} & 50.7 \scriptsize{$\pm$ 0.6} & \textbf{55.6} \scriptsize{$\pm$ 0.3} & 59.8 \scriptsize{$\pm$ 0.1} & 62.5 \scriptsize{$\pm$ 0.1} & 63.6 \scriptsize{$\pm$ 0.3} \\
  & FLOPS Prop  & 31.4 \scriptsize{$\pm$ 0.6} & 40.0 \scriptsize{$\pm$ 0.4} & 47.7 \scriptsize{$\pm$ 0.4} & 54.1 \scriptsize{$\pm$ 0.1} & \textbf{60.3} \scriptsize{$\pm$ 0.1} & \textbf{64.3} \scriptsize{$\pm$ 0.2} & \textbf{66.1} \scriptsize{$\pm$ 0.3} \\
  & Serving Rate (ours)  & 38.1 \scriptsize{$\pm$ 0.7} & 45.8 \scriptsize{$\pm$ 0.4} & \textbf{51.5} \scriptsize{$\pm$ 0.4} & \textbf{55.6} \scriptsize{$\pm$ 0.3} & 55.5 \scriptsize{$\pm$ 0.7} & 61.4 \scriptsize{$\pm$ 0.8} & 64.9 \scriptsize{$\pm$ 0.4} \\
  & Balanced Adj (ours)  & \textbf{42.8} \scriptsize{$\pm$ 0.9} & \textbf{46.7} \scriptsize{$\pm$ 0.3} & 51.1 \scriptsize{$\pm$ 0.4} & 53.3 \scriptsize{$\pm$ 1.1} & 56.9 \scriptsize{$\pm$ 1.3} & 61.7 \scriptsize{$\pm$ 0.8} & 64.3 \scriptsize{$\pm$ 0.8} \\
\bottomrule
\end{tabular}
\label{tab:biased-partition-results}
\end{table*}

\begin{table*}
\centering
\caption{Full experimental results for scenarios where nodes with the smallest models and datasets serve the majority of inference requests the CIFAR10 and CIFAR100 datasets. In this highly biased data partition, networks layers hold 3.4\%, 19.9\%, and 76.7\% of the data, respectively. All reported accuracy values are the mean value over three independent random seeds.}
\begin{tabular}{
  l  
  l  
  c  
  c  
  c 
}
\toprule
\multicolumn{2}{c}{} & \multicolumn{3}{c}{CIS Serving Rate Setting} \\
\cmidrule(lr){3-5}
Dataset & Strategy & {80-15-5} & {60-30-10} & {45-35-20} \\
\midrule
\multirow{3}{*}{CIFAR10} 
  & Equal Weight  & 36.5 \scriptsize{$\pm$ 4.0} & 45.5 \scriptsize{$\pm$ 3.2} & 56.6 \scriptsize{$\pm$ 3.0} \\
  & Serving Rate (ours)  & 41.3 \scriptsize{$\pm$ 2.7} & 40.1 \scriptsize{$\pm$ 2.2} & 55.3 \scriptsize{$\pm$ 2.5} \\
  & Serving Rate $p = 0.2$ (ours)  & \textbf{49.2} \scriptsize{$\pm$ 2.9} & \textbf{52.6} \scriptsize{$\pm$ 3.1} & \textbf{56.8} \scriptsize{$\pm$ 2.2} \\
\addlinespace  
\hline
\addlinespace  
\multirow{3}{*}{CIFAR100} 
  & Equal Weight  & 10.0 \scriptsize{$\pm$ 1.4} & 15.8 \scriptsize{$\pm$ 3.4} & 24.6 \scriptsize{$\pm$ 2.9} \\
  & Serving Rate (ours)  & 17.9 \scriptsize{$\pm$ 1.0} & 20.6 \scriptsize{$\pm$ 2.1} & 27.7 \scriptsize{$\pm$ 3.9} \\
  & Serving Rate $p = 0.2$ (ours)  & \textbf{30.1} \scriptsize{$\pm$ 2.1} & \textbf{30.8} \scriptsize{$\pm$ 1.2} & \textbf{30.9} \scriptsize{$\pm$ 0.9} \\
\bottomrule
\end{tabular}
\label{tab:highly-biased-partition-results}
\end{table*}

\end{document}